\documentclass[letterpaper, 10 pt, conference]{ieeeconf}  

\IEEEoverridecommandlockouts                              

\usepackage{amsmath} 
\usepackage{amssymb}  
\usepackage{graphicx}

\usepackage{xcolor}
\usepackage[english]{babel}
\usepackage{titlesec}

\usepackage{caption}
\usepackage{subcaption}

\newtheorem{theorem}{Theorem}[section]
\newtheorem{lemma}[theorem]{\bf{Lemma}}

\title{\LARGE \bf
DS-MPEPC: Safe and Deadlock-Avoiding Robot Navigation in Cluttered Dynamic Scenes
}

\author{Senthil Hariharan Arul$^{1}$, Jong Jin Park$^{2}$ and Dinesh Manocha$^{3}$
\thanks{$^{1}$Author is with Dept. of Electrical and Computer Engineering, University of Maryland, College Park, MD, USA.
        {\tt\small sarul1@umd.edu}}%
\thanks{$^{2}$Author is with Amazon Lab126, 1100 Enterprise Way, Sunnyvale, CA 94089, USA.
        {\tt\small jongpark@amazon.com}}%
\thanks{$^{3}$Author is with Dept. of Computer Science, University of Maryland, College Park, MD, USA.
        {\tt\small dmanocha@umd.edu}}%
}

\begin{document}

\maketitle
\thispagestyle{empty}
\pagestyle{empty}

\begin{abstract}
We present an algorithm for safe robot navigation in complex dynamic environments using a variant of model predictive equilibrium point control. We use an optimization formulation to navigate robots gracefully in dynamic environments by optimizing over a trajectory cost function at each timestep. We present a novel trajectory cost formulation that significantly reduces the conservative and deadlock behaviors and generates smooth trajectories. 
In particular, we propose a new collision probability function that effectively captures the risk associated with a given configuration and the time to avoid collisions based on the velocity direction. Moreover, we propose a terminal state cost based on the expected time-to-goal and time-to-collision values that helps in avoiding trajectories that could result in deadlock. We evaluate our cost formulation in multiple simulated and real-world scenarios, including narrow corridors with dynamic obstacles, and observe significantly improved navigation behavior and reduced deadlocks as compared to prior methods.

\end{abstract}

\section{Introduction}











Robots are gaining widespread use in everyday applications in the form of autonomous ride-sharing vehicles, robot vacuums, home monitoring robots, delivery robots, and urban surveillance drones. A key problem in all these applications is for the robot to move between different locations by navigating safely around static and dynamic obstacles. 

In this paper, we address the problem of computing safe paths for one or more robots in cluttered dense scenes. These include indoor pedestrian-rich environments such as homes, public places, and shopping malls, which can be particularly challenging scenes to navigate due to their highly dynamic, cluttered, and uncertain nature. First, pedestrian motion is unpredictable, and second, the robots have imperfect knowledge of their surroundings. Moreover, the robot may operate in spatially constrained regions such as narrow corridors, doorways, and spaces with multiple obstructions such as furniture, large objects, or pedestrians in the scene. Despite these challenges, the navigation algorithm must steer the robot safely, smoothly, and cooperatively around pedestrians and other obstacles.

The problem of autonomous navigation in large static or dynamic scenes has been well-studied. The main challenge is to generate trajectories that progress towards the goal, remain collision-free (safe) around obstacles, and are smooth. Many geometric or model-based methods~\cite{vo,rvo,orca,bvc} have been proposed for complex dynamic scenes consisting of one or more robots. Some model-based approaches optimize between these potentially conflicting navigation objectives, which can result in local minima or deadlocks. Thus, the robot may remain collision-free but may be deadlocked (or frozen) and unable to reach its goal successfully, even when a safe trajectory to the goal exists. Another class of methods is based on model predictive control (MPC)~\cite{mpc_orca,brito_mpcc}, which 
optimizes over a finite horizon to generate smooth, collision-free trajectories by incorporating predictions of agent and obstacle future behaviors. Its variant includes model predictive equilibrium point control (MPEPC) navigation~\cite{park_mpepc}, which formulates local navigation as a continuous, unconstrained, finite-horizon optimization problem by augmenting MPC with equilibrium point control (EPC~\cite{epc}). The underlying navigation formulation considers non-holonomic (e.g., wheelchair) dynamics and can handle static and moving obstacles to generate safe and smooth trajectories in dynamic environments. However, due to conflicting performance objectives, these methods can also result in a deadlock in certain scenarios.

\begin{figure}
    \centering
    \frame{\includegraphics[width=0.4\linewidth]{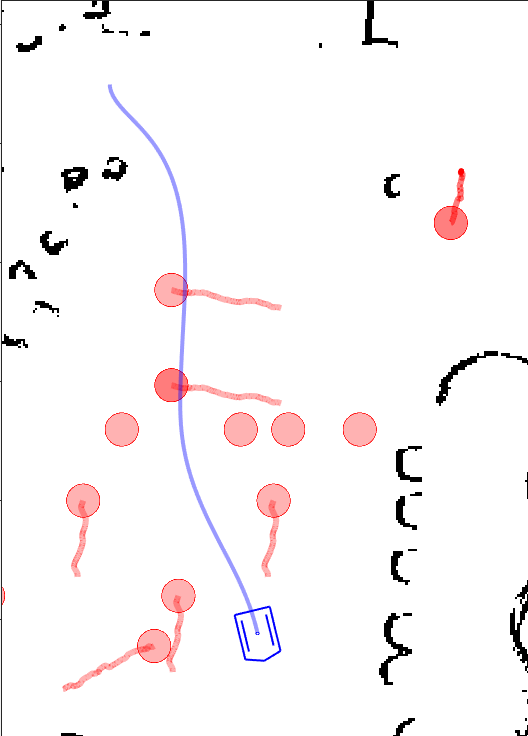}}
    \quad
    \frame{\includegraphics[width=0.42\linewidth]{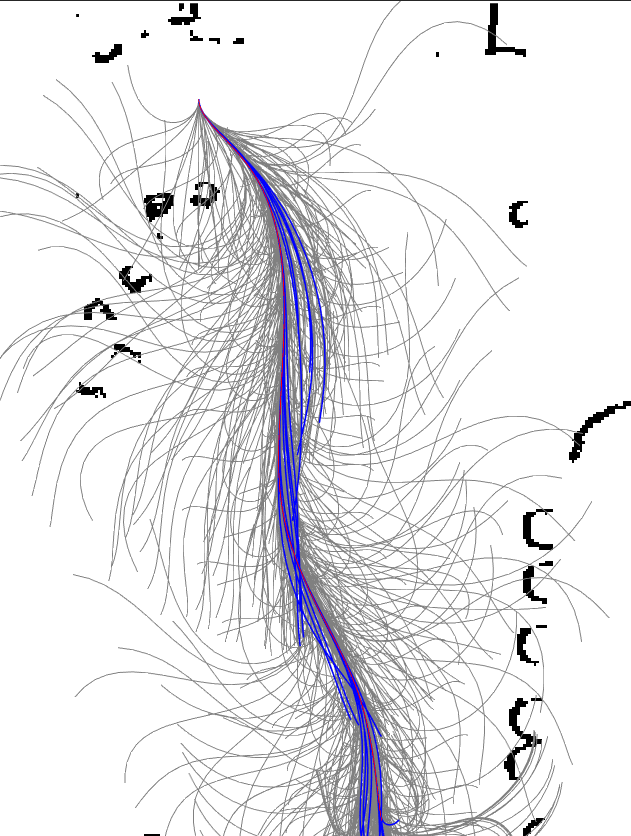}}
    \caption{An illustrative scenario showing a robot (blue) navigating a space with 12 pedestrians (red disks) with our modified cost function in real-world scenarios captured using sensors. Black regions denote static obstacles/ (Left) The red traces show the trajectory followed by the pedestrians over the past few timesteps computed using DS-MPEPC, and the blue trace shows the actual path followed by the agent. (Right) The figure illustrates the evaluated trajectories in this scenario, which are represented in gray. The blue trajectories are the optimal trajectory at each timestep.}
    \label{fig:my_label}
\end{figure}

Recently, learning-based methods have been used for navigation in dynamic environments. They can handle noisy sensor data and can work well in some environments.
Despite these advantages, learning-based methods lack safety guarantees and are non-explainable. 
Therefore, model-based approaches, owing to their safety and explainability, are still desirable for real-life navigation applications, though we need better methods to navigate in challenging scenarios. 

\subsection{Main Contributions}
In this paper, we present DS-MPEPC, an improved MPEPC-based ~\cite{park_mpepc} navigation algorithm. Our approach is designed to improve the performance in terms of reducing conservativeness and avoiding deadlocks (or to unfreeze), thereby resulting in improved navigation behavior. We present a modified trajectory cost formulation that improves the navigation in terms of reducing deadlocks, allowing agents to navigate a narrow passages, and can be used for multi-agent navigation in dynamic environments. The novel components of our work include:  

\begin{enumerate}
    \item A new collision probability formulation that captures the risk associated with a configuration state and the time available for the robot to avoid an impending collision. Our formulation is less conservative in terms of assigning a collision probability to a trajectory segment and results in improved navigation performance.

    \item A novel terminal cost term based on expected time-to-goal value. The terminal cost term preferentially chooses safe trajectories that reduce deadlocking behavior.

    \item We prove that optimizing against the modified cost function maintains the safety conditions that prevent the robot in a collision state from moving (Lemma~\ref{lemma:3}).
\end{enumerate}

We evaluate the proposed formulation in a variety of simulated scenarios with both static, dynamic, and multi-agent cases. We also simulate the robot motion in an environment generated from real-world sensor data as in MPEPC~\cite{park_mpepc} and demonstrate the benefits. The overall approach is fast and works well on challenging scenarios.


\section{Related Works}


\subsection{Navigation in Dynamic Environments}
Potential field methods~\cite{pfm} navigate robots by adding a repulsion field around obstacles and attracting them towards the goal but can suffer from local minima issues, oscillations, etc. There is considerable work on developing model-based methods based on velocity obstacles~\cite{vo} and their variants~\cite{rvo,orca}. VO~\cite{vo} presents an efficient method that defines a set of relative velocities between pairs of agents that cause collisions. At each time step, a velocity outside the VO is chosen for collision avoidance. RVO~\cite{rvo} extends the VO concept by assuming agents share equal responsibility for collision avoidance. In the ORCA algorithm~\cite{orca}, the RVO constraints are linearized to reduce the collision avoidance problem to that of linear programming. VO and their variants have shown good behavior in dynamic scenarios and are widely used. However, they generally consider simple velocity-controlled disk-shaped or elliptical~\cite{best2016real} agents, and constructing VO for arbitrarily shaped obstacles is non-trivial or results in very conservative behavior. In addition, planning a velocity at each time step makes generating smooth, acceleration-limited paths difficult. Furthermore, they can be overly conservative in environments in cluttered scenes~\cite{he2017efficient}. With ORCA, the linear constraints regard a half-space as invalid, and the linear program can be infeasible even for a small number of nearby obstacles~\cite{orca}. 

The VO concept was extended to a variety of agent dynamics, including double integrators~\cite{avo}, linear agents~\cite{lqr,lqg}, differential drive agents~\cite{orca_dd}, etc. However, these modifications either rely on augmenting the bounding geometry of the agents or linearizing the non-convex VO shape causing the formulation to be overly conservative. DWA~\cite{dwa} plans to avoid collisions in the velocity space. They are shown to be successful in low-speed scenarios but produce highly reactive behavior~\cite{brito_mpcc}. Buffered Voronoi cell (BVC)~\cite{bvc} presents an efficient method of computing collision-free trajectories by reasoning in the position space. Other techniques for collision avoidance evaluate trajectories against a cost function to optimize for a navigation plan between multiple agents based on a time-to-collision model~\cite{nh-ttc,cglr}. Inevitable collision states (ICS)~\cite{ics} computes a set of states that have no collision-free trajectories for an infinite time horizon. Though ICS provides a theoretical guarantee on collision avoidance, they are very conservative and could regard the entire workspace as forbidden. Trautman et al.~\cite{trautman} develop a probabilistic predictive model of cooperation and its importance for safe and efficient navigation in human crowds.

Model predictive control (MPC) has been used to generate smooth collision-free trajectories that show predictive navigation behavior. Cheng et al.~\cite{mpc_orca} employ ORCA constraints in an MPC framework, which reduces the velocity vibrations compared to ORCA. The feasible velocity set is defined by the ORCA constraints and can be infeasible due to ORCA's conservativeness. Brito et al.~\cite{brito_mpcc} present a model predictive contouring control (MPCC) frame that assumes a reference path and computes convex constraints on free space with predicted agent behaviors for generating local trajectories. Computing convex free space for arbitrary environments can be non-trivial and also conservative. Park et al.~\cite{park_mpepc} propose an MPEPC navigation framework that solves a finite horizon, unconstrained optimization to generate a smooth local trajectory.  
Optimizing for maximum progress at each timestep may not eventually lead the robot to its goal as the path can be obstructed by an obstacle like a piece of wall, and the robot gets deadlocked. The robot may need to move away or make limited progress toward the goal at certain timesteps to reach the goal eventually.

\subsection{Learning-based Navigation}
Recently, learning-based methods~\cite{cadrl,long} have been used for navigation in real-world dynamic scenes. In practice, they can handle sensor uncertainty in terms of better time-to-goal and success rate. DenseCAvoid~\cite{densecavoid} uses reinforcement learning and trajectory predictions to generate smooth, collision-free trajectories. Arpino et al.~\cite{rl_pednav} propose an RL-based method for robot navigation among pedestrians in real-life indoor environments. One major challenge is the lack of explainability and collision-free guarantees with such learning techniques. 

\subsection{Deadlock Resolution}
Deadlocks happen when the robot halts before reaching its goal and can be caused by local minimum decisions in the navigation problem due to a finite planning horizon. Frequently deadlocks are resolved using some heuristic rule. In BVC~\cite{bvc}, the robot detours along the edges of the Voronoi cell to resolve deadlocks. In V-RVO~\cite{vrvo}, the method proposed a simple communication-based strategy for deadlock resolution. In~\cite{wallfollowing}, the planner detects a deadlock and performs wall following to resolve deadlocks. In this paper, we define a terminal state cost function based on the time-to-goal and time-to-collision values, which shows deadlock-resolving behavior in our test scenarios.



\section{Background}\label{sec:3}
In this section, we briefly describe the MPEPC navigation~\cite{park_mpepc} algorithm. We also give an overview of prior work on smooth control law~\cite{park_smoothlaw}. 

\subsection{Assumptions}

We consider a differential-drive robot navigating complex environments with static and dynamic obstacles. The static environment information is available as an occupancy map, and we assume the current position and velocity data can be estimated for dynamic obstacles (like pedestrians). In addition, we use a constant velocity model to estimate the future positions of dynamic obstacles over the planning horizon.  Our approach can be used for simulated environments as well as real-world datasets captured using visual sensors.

\begin{table}[]
    \centering
    \begin{tabular}{|c|c|}
        \hline
        $p_c$ & Collision Probability \\
        \hline
        $p_s$ & Survivability\\
        \hline
        $v_i$ & Velocity at time step $i$\\
        \hline
        $v_{max}$ & Max. velocity\\
        \hline
        $d_o$ & Distance to closest obstacle\\
        \hline
        $d_g$ & Distance to goal \\
        \hline
        $h$ & Time step \\
        \hline
        $TTC$ & Time to collision \\
        \hline
        $TTG$ & Time to goal\\
        \hline
        $T$ & Target goal \\
        \hline
        $r$ & Robot's distance from T\\
        \hline
        $los$ & Line of sight from robot to T\\
        \hline
        $\theta$ & Orientation of T relative to $los$\\
        \hline
        $\delta$ & Orientation of robots pose with $los$\\
        \hline
        $z*$ & Trajectory parameter $(r,\theta,\delta,v_{max})$\\
        \hline
        $N$ & Terminal timestep\\
        \hline
        $\omega$ & Angular velocity\\
        \hline
    \end{tabular}
    \caption{Symbols and Notation}
    \label{tab:notation}
\end{table}

\subsection{Smooth Control Law}\label{sec:smoothlaw}

Park and Kuipers~\cite{park_smoothlaw} define an ego-centric coordinate system relating the robot's current pose with its target goal pose $T$. Consider a goal configuration $T$ at a distance $r$ away from the robot, let $\theta$ represent the orientation of $T$ relative to the line of sight from the robot to the target, and let $\delta$ define the orientation of the robot's current configuration with the line of sight. The triplet $(r, \theta, \delta)$ defines an ego-centric coordinate system describing the robot's current pose with its target goal pose.

Further, the authors define a non-linear pose-following control law that globally drives a robot towards $T$. Here, $k_1$ and $k_2$ define gain parameters. 
$$
\omega = -\frac{v}{r} [ k_2 (\delta - arctan(-k_1\theta) + (1 + \frac{k_1}{1 + (k_1\theta)^2})\sin \delta]
$$
$$
\omega = \mathbf{\kappa} v
$$
The above control law defines the shape of the trajectory with the target $T$ acting as an attractor; the maximum velocity $v_{max}$ defines strength of the attraction.

The ego-centric frame $(r, \theta, \delta)$ and $v_{max}$ define a parameterized trajectory space. The 4-dimension vector, represented by $z* = (r, \theta, \delta, v_{max})$ completely defines the trajectory of the robot to target $T$. The trajectory space is smooth and realizable by construction. Thus, given a trajectory parameter $z*$, we can completely define the trajectory converging to the target.

\subsection{Robot Navigation using MPEPC}\label{ref:cost_mpepc}

Park et al.~\cite{park_mpepc} frames the safe navigation as an optimization problem, selecting a suitable trajectory parameter $z*$, which generates the desired trajectory. 
Let 
$$
q_{z*} : [0,T]\rightarrow C 
$$
denote the trajectory parameterized by $z*$ within a finite horizon $T$.  

The planner optimizes to select a suitable trajectory parameter $z*$ which minimizes the expected trajectory cost.
\begin{multline}\label{eqn:orig_cost}
    J(q_{z*}) = \sum^N  p_{s_i} * J_{progress_i} + J_{action_i} \\+ (1-p_{s_i})* J_{collision_i}
\end{multline}
The term $J_{progress}$ captures the progress the robot makes towards the goal, $J_{action}$ captures the cost of applying large action, and $J_{collision}$ captures the cost of collision. The progress and collision terms are weighed by survivability ($p_s$), which is a measure of the probability that the trajectory remains collision-free. $N$ is the planning horizon.
 
The authors define survivability based on the notion of collision probability into the cost function. From~\cite{park_thesis}, the probability of collision for a robot over a short time segment is defined by the bell-shaped function:

\begin{equation}\label{eqn:pci}
p_{c_i} = \exp(-d_o^2/ \sigma^2)
\end{equation}

The probability of survivability is defined using the $p_c$ as
\begin{equation}\label{eqn:psi}
    p_{s_i} = \Pi_1^i (1 - p_{c_k}).
\end{equation}
{\color{black}{
This formulation is used in a navigation framework~\cite{park_thesis}. The framework guarantees probabilistic safety rather than trying to provide absolute collision avoidance, which is very hard to guarantee in cluttered, dynamic environments. 
}}
\section{Our Method: DS-MPEPC}
In this section, we describe our algorithm for robot navigation. This is based on our new collision probability function $\tilde{p}_{c_i}$ (Section~\ref{ref:coll_prob}) which is used to define the survivability term $\tilde{p}_{s_i}$ (as in Equation~\ref{eqn:psi}). Moreover, we define an additional terminal state cost $J_{terminal}$ (Section~\ref{ref:term_cost}) based on time-to-goal and time-to-collision values.
The modified cost function is given below:
\begin{multline}\label{eqn:mod_cost}
\tilde{J}(q_{z*}) = \\ \sum^N \big( \tilde{p}_{s_i} * J_{progress_i} + J_{action_i} + (1 - \tilde{p}_{s_i}) * J_{collision_i} \big) \\+ J_{terminal}(x_N)    
\end{multline}

\subsection{Collision Probability ($\tilde{p}_{c_i}$)}\label{ref:coll_prob}
As seen in Equation~\ref{eqn:pci}, the collision probability is defined as a function of the robot's distance to its closest obstacle ($d_o$). The collision probability captures the risk associated with a particular configuration and is a bell-shaped curve that raises to $1$ when the $d_o = 0$, and $0$ when $d_o = \infty$. Our proposed collision probability formulation aims to be less conservative in terms of assigning a collision probability to a trajectory segment. 

A robot almost in collision with a wall or an obstacle has a survivability of zero. From Equation~\ref{eqn:orig_cost}, we observe the progress cost is scaled down to zero by the survivability term. Consider a goal located at a short distance (e.g., $2$m) in front of the robot. Since the progress cost is scaled down to zero, the planner does not select a trajectory that leads the robot to the goal even when it increases the distance to the wall and is safe. A purely distance-based collision probability does not capture the ability of the robot to reduce the {\em{risk}} by moving away from the obstacle.

We modify the collision probability function based on the intuition that, for a robot with close obstacle proximity, it is safer to move away from the obstacle rather than toward it. We consider the following behavior for the modified $p_c$ function.
\begin{itemize}
    \item $p_c$ increases as the robot comes close to an obstacle.
    \item $p_c$ values reduce along motion directions with higher time to collision. 
\end{itemize}

Hence, our modified collision probability is a function of distance to the closest obstacle (a reactive term) and time to collision (an anticipatory term). Time to collision (TTC) is defined as the number of seconds in the future before a collision if the agent and obstacle continue their direction of motion. A TTC value of $0$ indicates an agent is already in collision, while a TTC value of $\infty$ indicates the agent and the obstacle will not collide if they follow their current velocities forever.

The modified collision probability has two terms as  represented below:
\begin{equation*}
\tilde{p}_c = \texttt{Reactive term} * \texttt{Anticipatory term}
\end{equation*}

The reactive component captures the {\em{risk}} associated with being at a particular configuration given by the 2-D position and orientation. The reactive term is distance-based and physically represents the localization uncertainty. Thus this terms is same as $p_{c}$.  

The anticipatory component is a function of time-to-collision (TTC), which provides a look ahead at the time available to the robot to avoid collision based on the velocity.

\begin{equation}\label{eqn:mod_pc}
\tilde{p}_c = \exp{\bigg(-\frac{d_o^2}{\sigma_{d}^2}\bigg)} * \bigg(1 - a\exp{\bigg( \frac{(1/TTC)^2}{\sigma_{1/TTC}^2} \bigg)} \bigg)
\end{equation}

From Equation~\ref{eqn:mod_pc}, we see that the anticipatory component only reduces the effect of the distance-based term provided the robot moves away from the obstacle based on the velocity direction. Thus, the anticipatory component cannot increase the collision probability.
\\\\
\subsubsection{Weight parameter $a$}
A robot in a narrow passage moving parallel to the walls has a time-to-collision value of $\infty$. The robot is not heading towards the obstacle, but it could be unsafe to regard the collision probability as $0$ in tight spaces. We introduce a weight parameter $a$ that limits the effect of the time-to-collision term. With a weight of $a < 1$, the time-to-collision component merely reduces the effect of the distance-based term but does not nullify it. Thus, the following Lemma holds:

\begin{lemma}\label{lemma:1}
The proposed collision probability ($\tilde{p}_{c_i}$) results in a survivability ($\tilde{p}_{s_i}$), which is at least the original survivability ($p_{s_i}$). Thus, it is less conservative in terms of regarding a trajectory segment as survivable.
\end{lemma}
\begin{proof}
From Equation~\ref{eqn:mod_pc}, the maximum and minimum values of $\tilde{p}_{c_i}$ in relation to $p_{c_i}$ is given by,
$$
\max \tilde{p}_{c_i} = p_{c_i}, \quad \min \tilde{p}_{c_i} = (1 - a) \cdot p_{c_i}.
$$
From the definition of survivability, $\tilde{p}_{s_i} = \Pi_1^i (1 - {p}_{c_k})$.
From the maximum values of $\tilde{p}_{c_i}$, we know $\min (1 - \tilde{p}_{c_i}) = (1 - p_{c_i})$. Thus, the minimum values of $\tilde{p}_{s_i}$ are bounded by
$$\min \tilde{p}_{s_i} = \Pi_1^i (1 - {p}_{c_k}) = p_{s_i}.$$
\end{proof}

\begin{lemma}\label{lemma:2}
For a robot in collision state $\tilde{p}_{s_i} = 0, \forall i$
\end{lemma}
\begin{proof}
    When the robot is in collision state (i.e., $d_o = 0$), the TTC value is $0$ for any non-zero relative velocity between the robot and the obstacle. Hence, in Equation~\ref{eqn:mod_pc} the term
    $\big(1 - a\exp{\big( \frac{(1/TTC)^2}{\sigma_{1/TTC}^2} \big)} \big) = 1$. Thus, $\tilde{p}_{c_1} = p_{c_1} = 1 \implies \tilde{p}_{s_1} = 0$. From Equation~\ref{eqn:psi}, we get $\tilde{p}_{s_i} = 0, \forall i$.
\end{proof}

\begin{figure*}[h!]
\centering
\begin{subfigure}{0.24\textwidth}
  \centering
  \includegraphics[trim={5cm 5cm 0 2cm},clip, width=.95\linewidth]{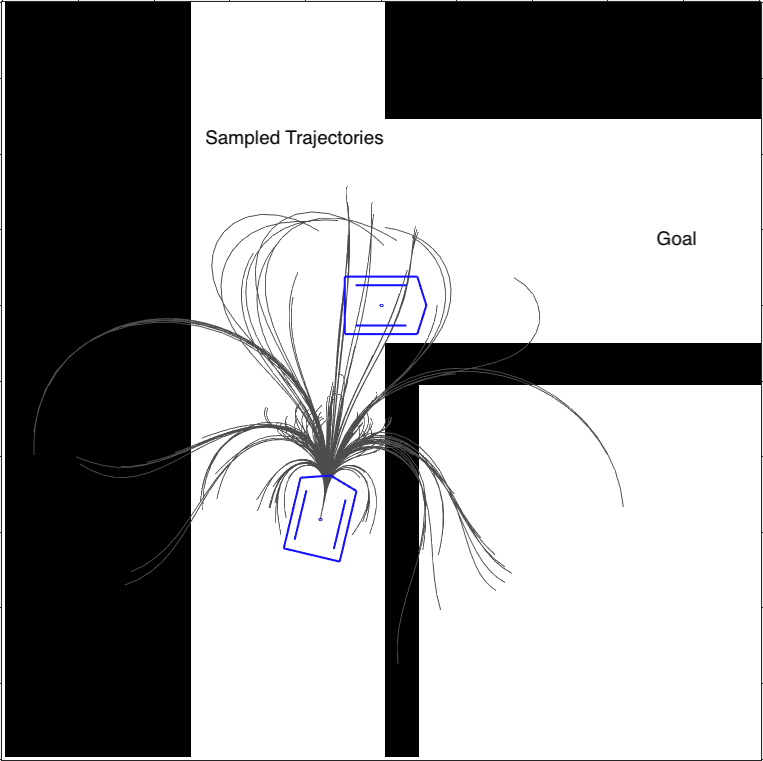}
  \caption{Sampled Trajectories}
  \label{fig:sample}
\end{subfigure}
\begin{subfigure}{0.24\textwidth}
  \centering
  \includegraphics[trim={5cm 5cm 0 2cm},clip,width=.95\linewidth]{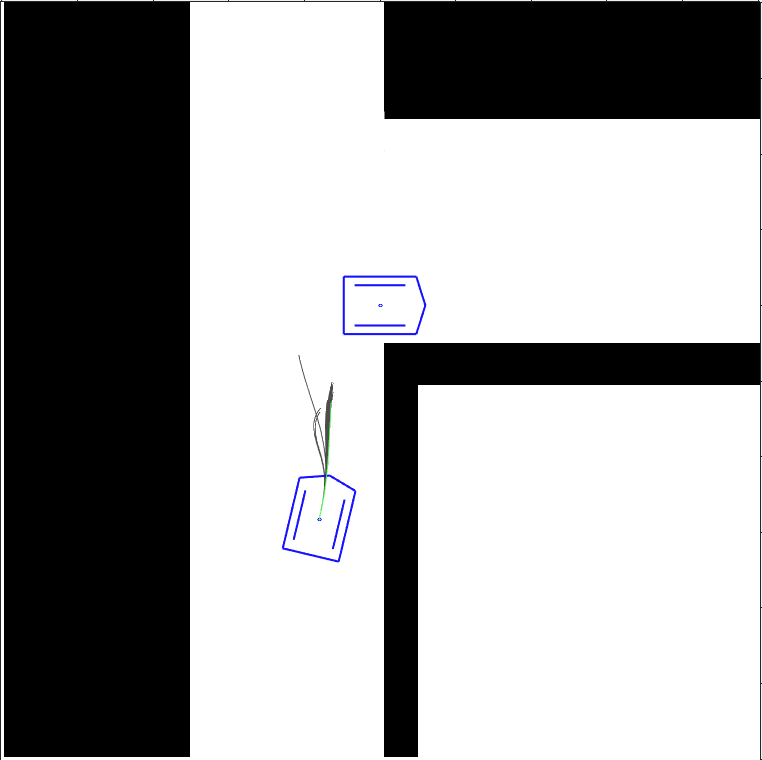}
  \caption{MPEPC Cost}
  \label{fig:cost}
\end{subfigure}
\begin{subfigure}{0.24\textwidth}
  \centering
  \includegraphics[trim={5cm 5cm 0 2cm},clip,width=.95\linewidth]{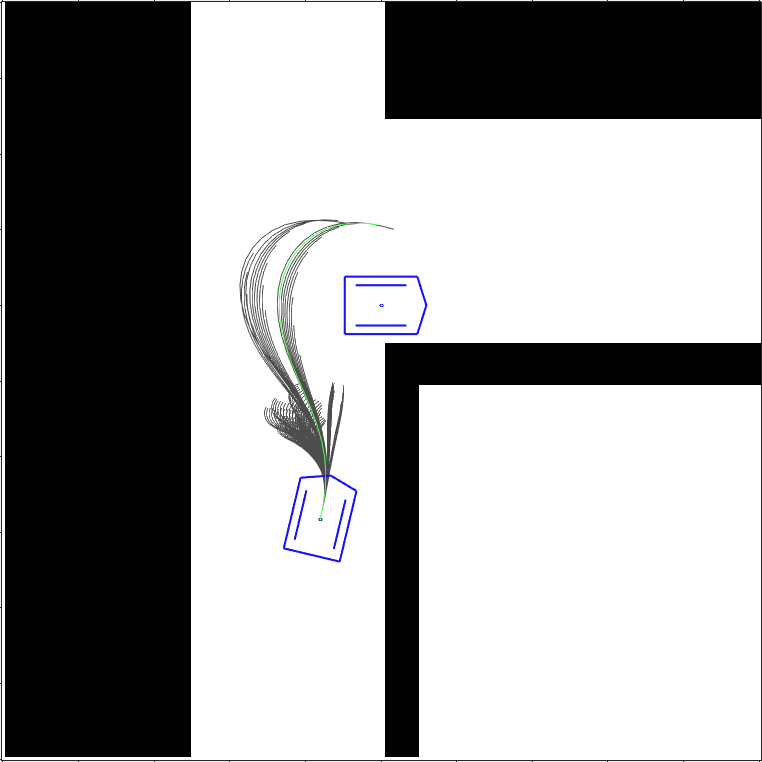}
  \caption{Time-to-Goal}
  \label{fig:ttg}
\end{subfigure}
\begin{subfigure}{0.24\textwidth}
  \centering
  \includegraphics[trim={5cm 5cm 0 2cm},clip,width=.95\linewidth]{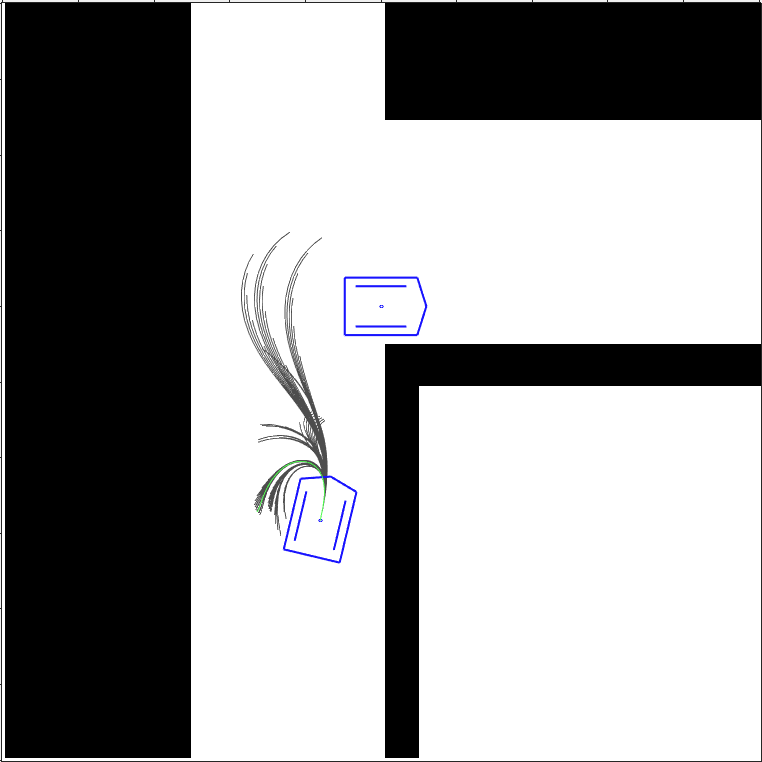}
  \caption{Time-to-Collision}
  \label{fig:ttc}
\end{subfigure}
\caption{We illustrate a scenario with two robots in a T-shaped corridor. One robot is stationary, and the other robot is moving toward the goal by turning into the corridor. (a) We show a few sampled trajectories for the robot moving toward the goal. (b) We illustrate 50 trajectories with the lowest cost as computed with the MPEPC cost function. We observe the trajectory distribution is directed toward the other robot, and would eventually lead to a deadlock before reaching the goal. In (c), we show a distribution of 50 trajectories with the lowest TTG values. In (d), we show a distribution of 50 trajectories with the highest TTC value. From (c) and (d) we notice the distribution of trajectories allows the robot to move around the stationary robot and TTC and TTG values provide a metric to reduce deadlocks.}
\label{fig:ttcttg_Eval}
\end{figure*}

\subsection{Terminal Cost ($J_{terminal}$) }\label{ref:term_cost}
The cost formulation optimizes for maximum progress made toward the goal while maintaining safety. The MPEPC cost may cause the robot to deadlock in front of an obstacle because a given local trajectory makes the maximum progress towards the goal while being safe. In some scenarios, moving away from the goal can eventually steer the robot towards the goal. We incorporate this behavior into the cost formulation by introducing a terminal state cost. We define the terminal cost as a function of the agent's expected time-to-goal and time-to-collision as computed at the terminal state of the local trajectory.
In figure~\ref{fig:ttcttg_Eval}, we illustrate a scenario to show time-to-goal and time-to-collision values are suitable for deadlock reduction.
\\
{\subsubsection{Expected Time-to-Goal:}}
The expected time-to-goal is computed as the ratio of the euclidean distance to the goal from the terminal trajectory state and the component of the terminal state velocity in the direction of the goal vector. Hence, if the terminal velocity tends to zero or if the terminal velocity has no components in the goal direction, the expected time-to-goal tends to infinity. Consequently, when the distance to the goal tends to zero, the time-to-goal tends to zero.
When a robot is deadlocked (or frozen), its distance to the goal is non-zero, its velocity is zero, and the time-to-goal is $\infty$. Hence, time-to-goal provides a relevant metric for deadlock detection. 
\\
{\subsubsection{Expected Time-to-Collision:}}
The time-to-collision value depends on the agent's and obstacle's velocities. An agent halting before a static robot would have a time-to-collision of $\infty$, indicating the agent is safe. 
A robot halting and facing an obstacle may have all its future trajectories passing through the obstacle due to its orientation, while the robot looking into free space may have more collision-free trajectories to optimize.
To incorporate the idea, we compute the TTC at the terminal state, assuming the robot travels by maintaining its terminal state orientation with its max velocity. This approximate computation of the time-to-goal is suitable, as it captures the terminal state orientation of the robot.
\\
{\subsubsection{Computing $J_{terminal}$:}}

We define a novel terminal state cost based on the time-to-goal and time-to-collision values, which aid in reducing deadlocks and facilitate the robot reaching the goal. $J_{terminal}$ has three main terms. First, the terminal state survivability $p_{s_N}$ ensures the $J_{terminal}$ shrinks to zero when the trajectory collides. Second, the $C_{TTG}$ term is such that $C_{TTG} \rightarrow 1$ as $TTG \rightarrow \infty$. The third term $C_{TTC} \rightarrow 1$ as $TTC \rightarrow \infty$. Thus the product $C_{TTG} * C_{TTC}$ acts on trajectories with $TTG \rightarrow \infty$, and it prefers trajectories with longer time-to-collision values. As mentioned above, a configuration with a higher time-to-collision may have more trajectories to choose from in the subsequent timestep, thus aiding in resolving a deadlock. Moreover, $J_{terminal}$ has negative values only on trajectories with non-zero survivability to ensure it works only on non-colliding trajectories. We define the terminal state cost as follows:
$$
C_{TTG} = \exp \bigg(-\frac{{1}/{TTG}^2}{\sigma_{1/TTG}^2}\bigg),
$$
$$
C_{TTC} = \exp \bigg(- \frac{1/TTC^2}{\sigma_{1/TTC}^2} \bigg),
$$
\begin{equation}
    J_{terminal} = - p_{s_N} * \bigg( C_{TTG} * C_{TTC} \bigg).
\end{equation}
As can be observed, the terminal state cost ($J_{Terminal}$) is bounded $J_{terminal} = [-1, 0]$.
\begin{lemma}\label{lemma:3}
 The modified cost function maintains the MPEPC property that avoids moving when robot is already in a collision state.
\end{lemma}
\begin{proof}
    When the robot is in collision state ($d_o$), the survivability $\tilde{p}_{s_i} = 0, \; \forall i$ from Lemma~\ref{lemma:2}. Under this survivability condition, the modified cost definition (Equation~\ref{eqn:mod_cost}) reduces to the MPEPC's cost definition (Equation~\ref{eqn:orig_cost}).
\end{proof}

\subsection{Trajectory Optimization}
The navigation optimization problem selects a suitable trajectory and control input by optimizing with our modified cost function. From Section~\ref{ref:cost_mpepc}, we know the navigation problem is framed as an unconstrained finite horizon optimization problem to select a trajectory parameter $z*$. Given a $z*$ the trajectory is completely defined (Section~\ref{sec:smoothlaw}). Thus, for a $z*$, the trajectory is simulated and the cost is computed with our proposed cost function. The optimization minimizes $\tilde{J}(q_{z^*})$ (Equation~\ref{eqn:mod_cost}) to compute an optimal $z*$ having the minimum trajectory cost.

\begin{equation*}
\underset{z^*}{\text{minimize}} \quad \tilde{J}(q_{z^*})   
\end{equation*}

The robot's control inputs in terms of linear and angular velocity is computed from the optimal $z*$ as described in Section~\ref{sec:smoothlaw}.




\section{Evaluations}
In this section we describe our evaluation setup and highlight the performance on challenging benchmarks.

\begin{figure*}
\centering
\begin{subfigure}{0.19\textwidth}
  \centering
  \includegraphics[width=.9\linewidth]{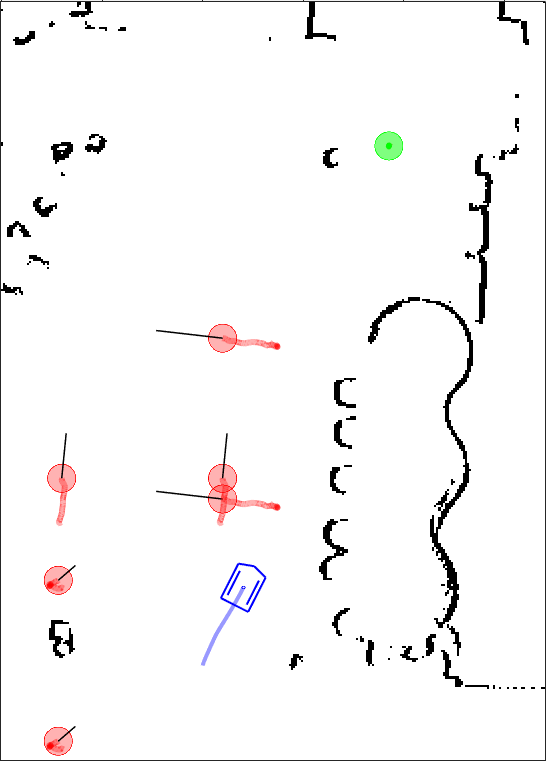}
  \caption{t = 5s}
\end{subfigure}%
\begin{subfigure}{0.19\textwidth}
  \centering
  \includegraphics[width=.9\linewidth]{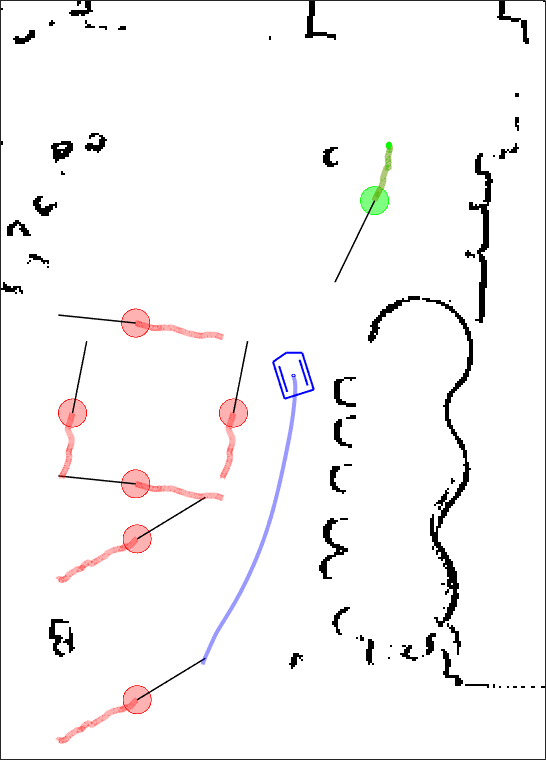}
  \caption{t = 10s}
\end{subfigure}
\begin{subfigure}{0.19\textwidth}
  \centering
  \includegraphics[width=.9\linewidth]{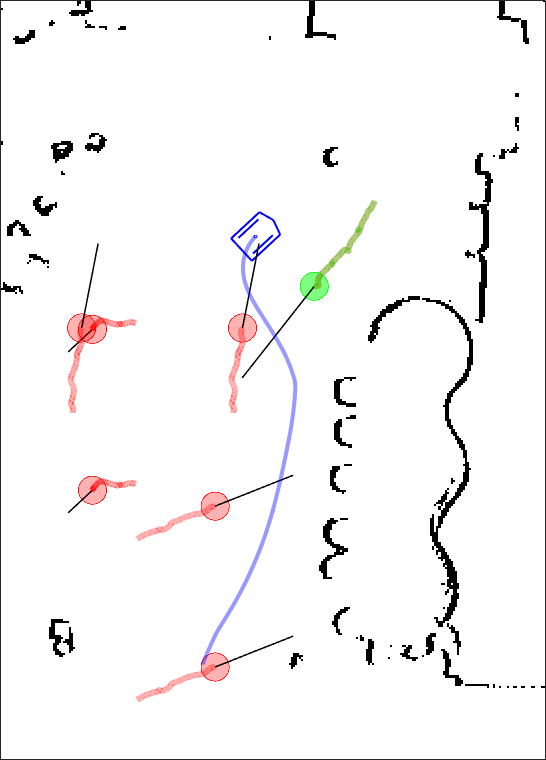}
  \caption{t = 15s}
\end{subfigure}
\begin{subfigure}{0.19\textwidth}
  \centering
  \includegraphics[width=.9\linewidth]{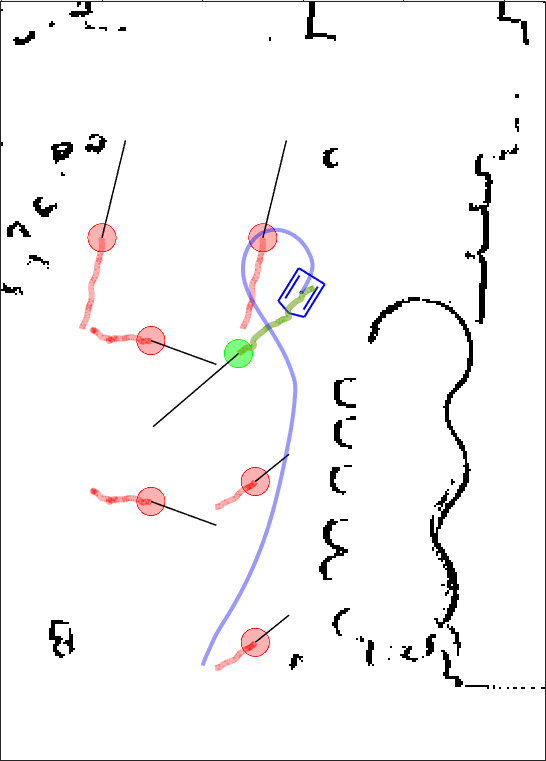}
  \caption{t = 20s}
\end{subfigure}
\begin{subfigure}{0.19\textwidth}
  \centering
  \includegraphics[width=.9\linewidth]{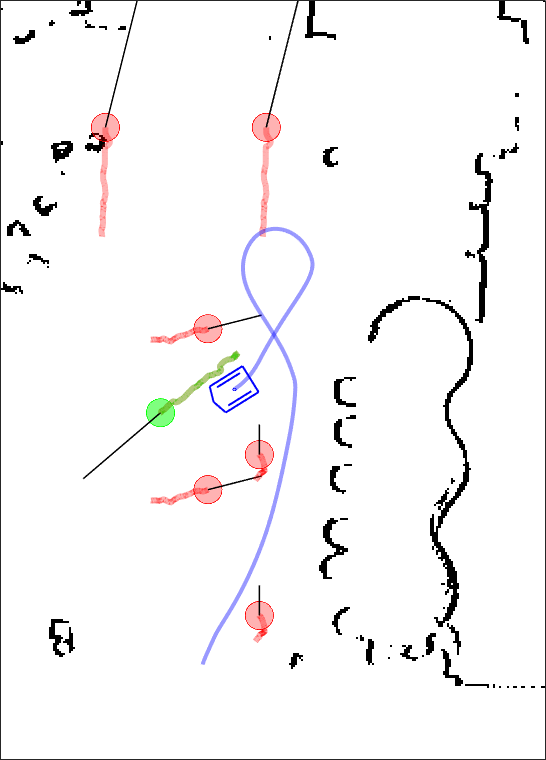}
  \caption{t = 25s}
\end{subfigure}
\caption{{\bf Evaluation in real-world environments: Hall scenario}: The figure illustrates a hallway scenario with multiple pedestrians denoted by red and green disks.  This noisy data is captured using sensors The static obstacles in the environment are represented in black by the occupancy map. The robot (blue) follows a pedestrian denoted by a green disk while navigating around static (black) and other dynamic obstacles (red). The red and green traces show the trajectory followed by different pedestrian over the past few timesteps. The black line shows the predicted future trajectory of the pedestrian. The blue trace shows the trajectory of the robot during the navigation. The figure from left to right show the navigation simulation at regular time intervals. We observe DS-MPEPC is able to compute a smooth path around the obstacles.}
\label{fig:hall}
\end{figure*}

\begin{figure*}
\centering
\begin{subfigure}{0.32\textwidth}
  \centering
  \includegraphics[width=.95\linewidth]{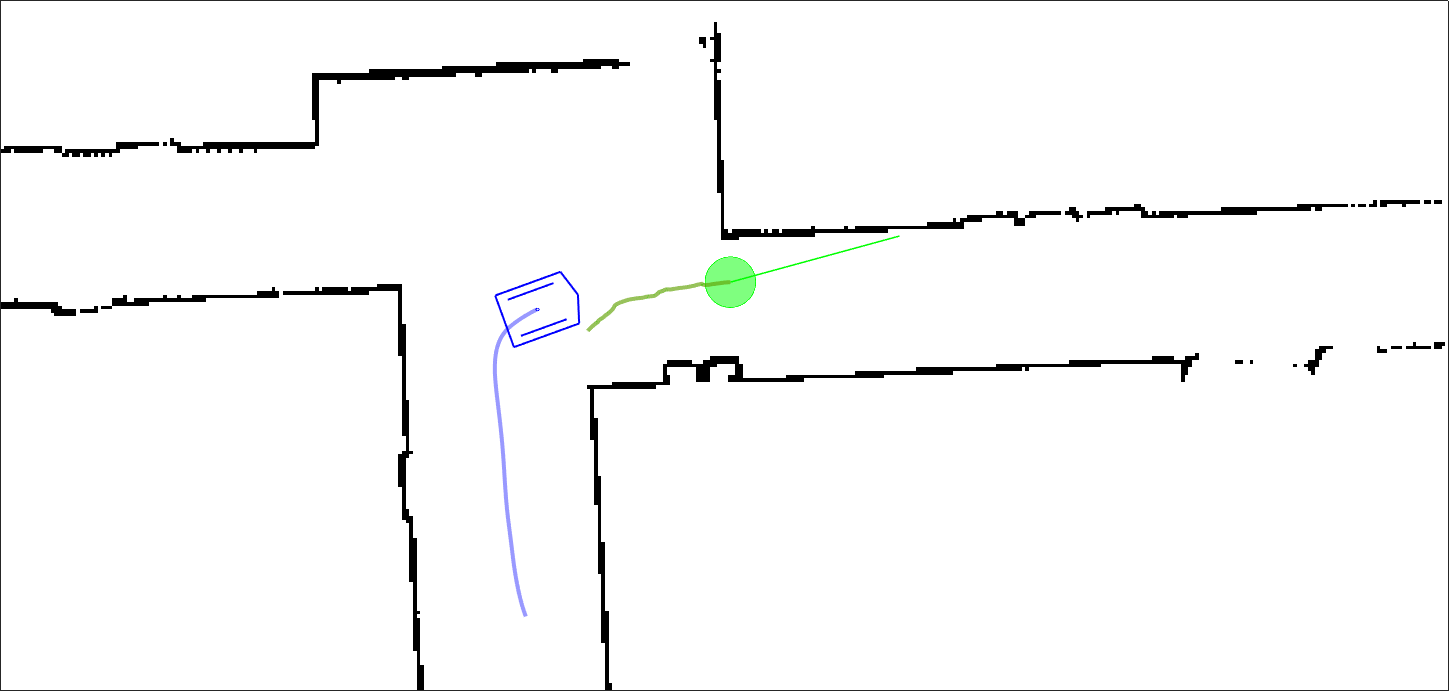}
  \caption{A subfigure}
\end{subfigure}%
\begin{subfigure}{0.32\textwidth}
  \centering
  \includegraphics[width=.95\linewidth]{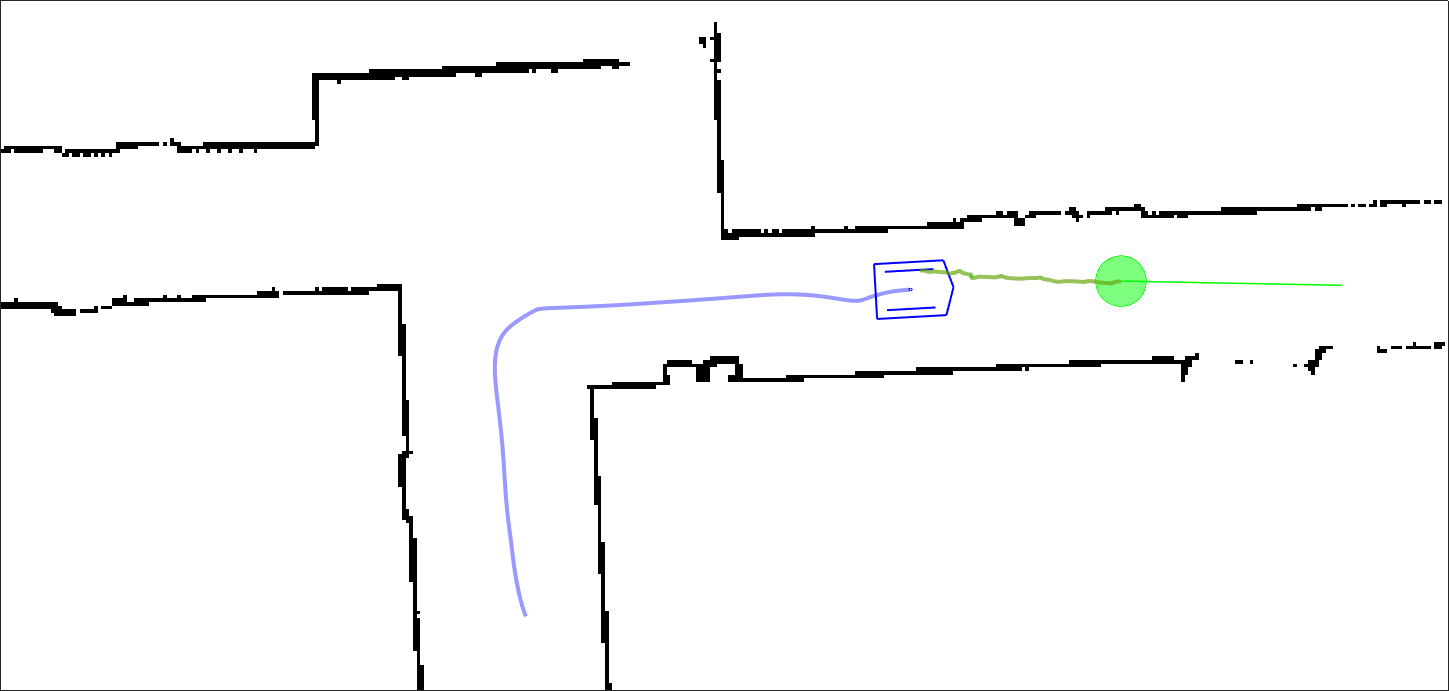}
  \caption{A subfigure}
\end{subfigure}
\begin{subfigure}{0.32\textwidth}
  \centering
  \includegraphics[width=.95\linewidth]{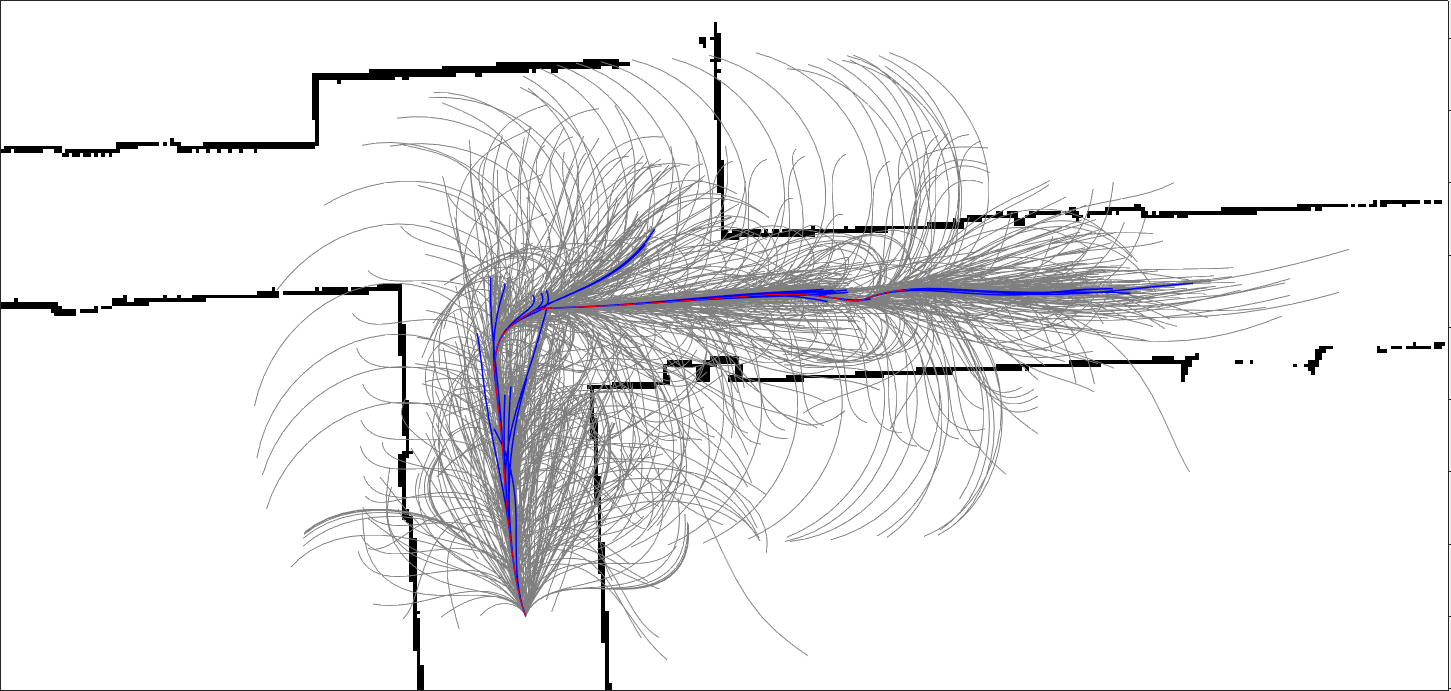}
  \caption{A subfigure}
\end{subfigure}
\caption{{\bf Evaluation in real-world environments: L-shaped corridor}: We highlight the performance in another real-world scene captured using sensors. The figure illustrates a scenario with the robot (blue) following a pedestrian (green disk) into a narrow corridor. Figures (a) and (b) show the actual trajectory followed by the robot during the simulation. Figure (c) shows the evaluated trajectories (gray) by the planner during the simulation, while the optimal trajectory at each planning cycle is denoted in blue. DS-MPEPC is able to compute a smooth trajectory in this challenging scenario.}
\label{fig:lcorridor}
\end{figure*}

\begin{figure*}
\centering
\begin{subfigure}{0.22\textwidth}
  \centering
  \includegraphics[width=.95\linewidth]{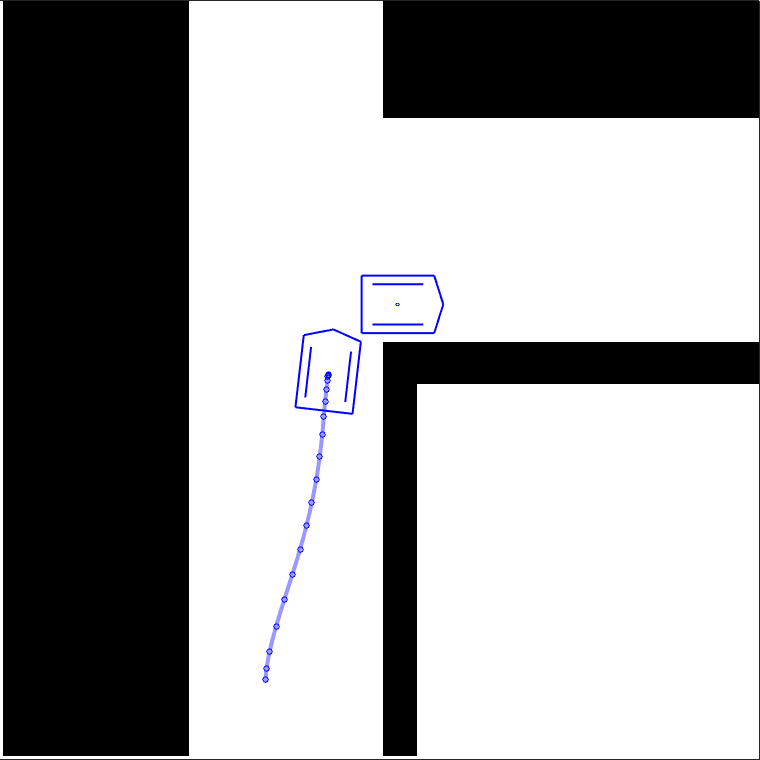}
  \caption{Actual robot trajectory}
\end{subfigure}%
\begin{subfigure}{0.22\textwidth}
  \centering
  \includegraphics[width=.95\linewidth]{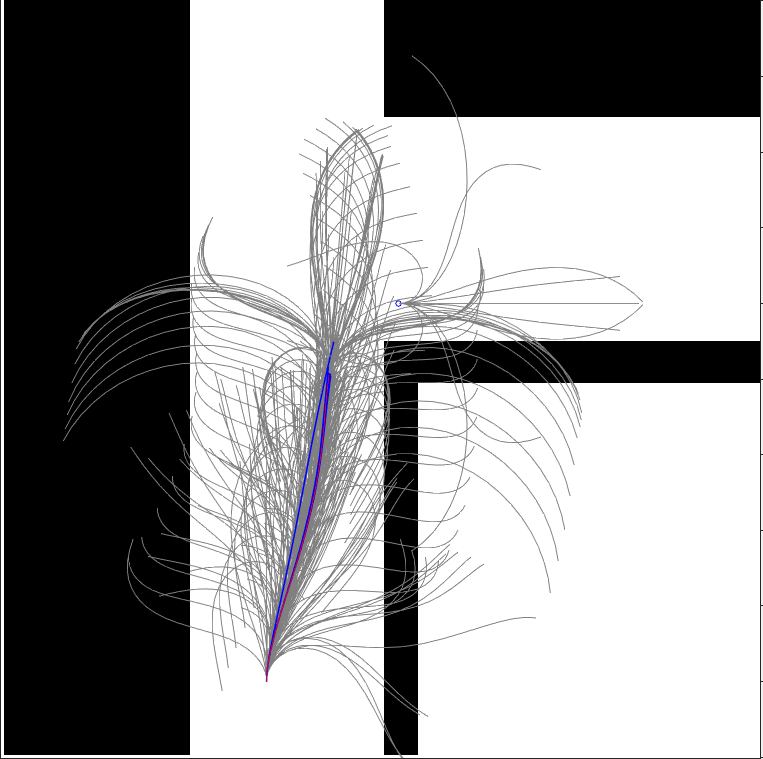}
  \caption{Evaluated trajectories}
\end{subfigure}
\quad \quad \quad
\begin{subfigure}{0.22\textwidth}
  \centering
  \includegraphics[width=.95\linewidth]{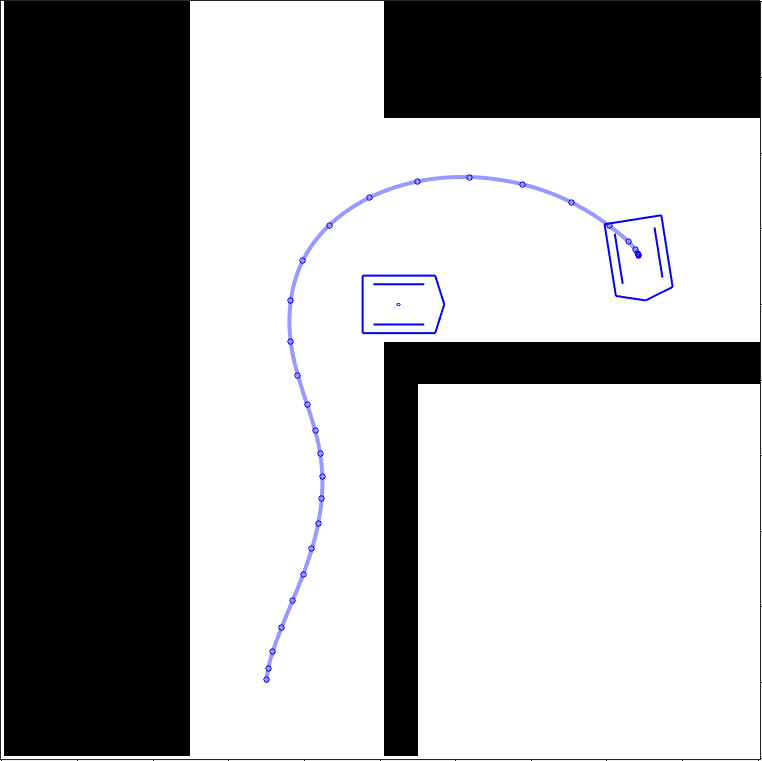}
  \caption{Actual robot trajectory}
\end{subfigure}
\begin{subfigure}{0.22\textwidth}
  \centering
  \includegraphics[width=.95\linewidth]{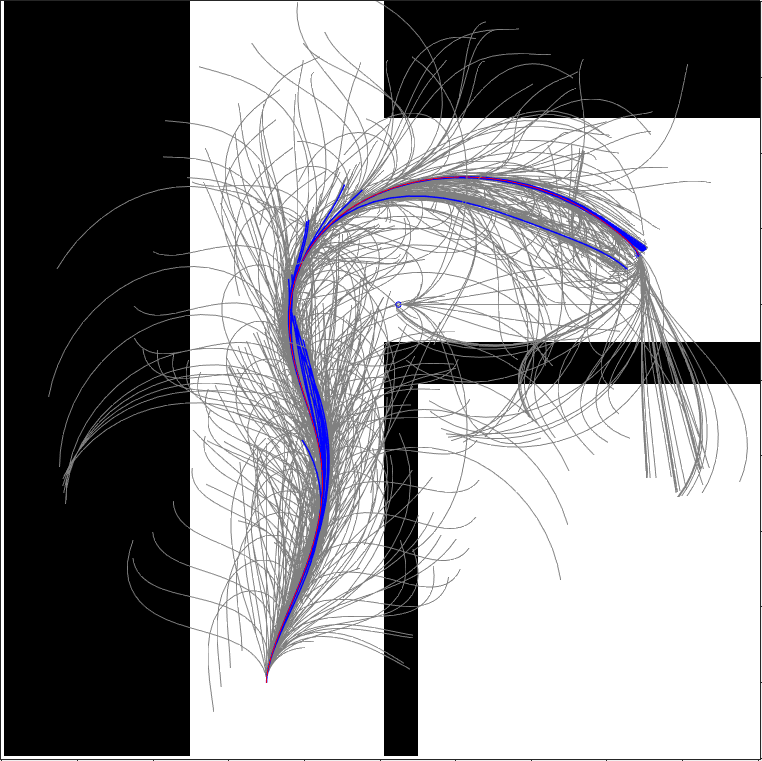}
  \caption{Evaluated trajectories}
\end{subfigure}
\caption{{\bf T-Corridor}: This scenario considers two robots, one moving from the bottom and turning into the corridor on the right. The other robot is stationary and obstructs this corridor at its entrance. Figures (a) and (b) show the robot trajectory (blue) and evaluated trajectories for the moving robot while using the original MPEPC cost function. We can observe the robot deadlocks. Figures (c) and (d) show the robot trajectory while using our proposed cost modification. In this case, the robot navigates successfully without a deadlock with DS-MPEPC.}
\label{fig:tcorridor}
\end{figure*}

\begin{figure*}
\centering
\begin{subfigure}{0.15\textwidth}
  \centering
  \includegraphics[height=5cm,width=.75\linewidth]{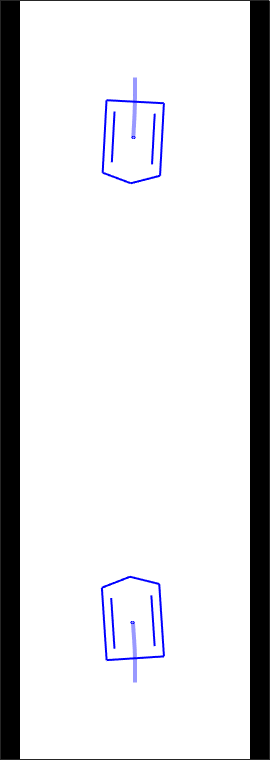}
  \caption{t=2s}
\end{subfigure}%
\begin{subfigure}{0.15\textwidth}
  \centering
  \includegraphics[height=5cm,width=.75\linewidth]{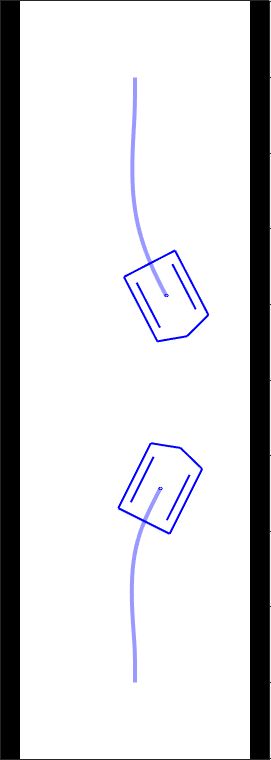}
  \caption{t=4s}
\end{subfigure}
\begin{subfigure}{0.15\textwidth}
  \centering
  \includegraphics[height=5cm,width=.75\linewidth]{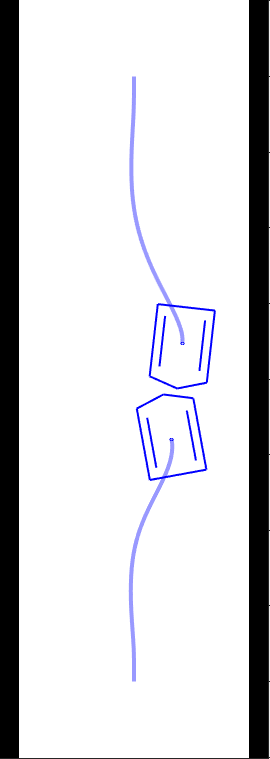}
  \caption{t=6s}
\end{subfigure}
\quad \quad \quad
\begin{subfigure}{0.15\textwidth}
  \centering
  \includegraphics[height=5cm,width=.75\linewidth]{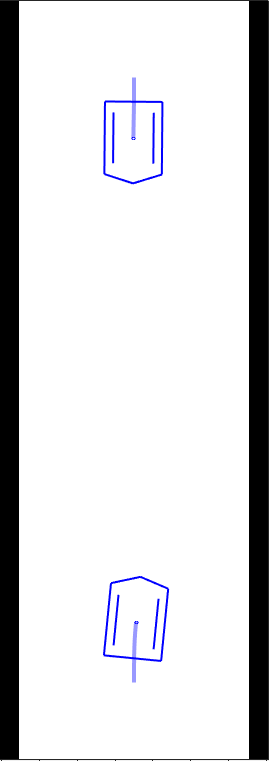}
  \caption{t=2s}
\end{subfigure}%
\begin{subfigure}{0.15\textwidth}
  \centering
  \includegraphics[height=5cm,width=.75\linewidth]{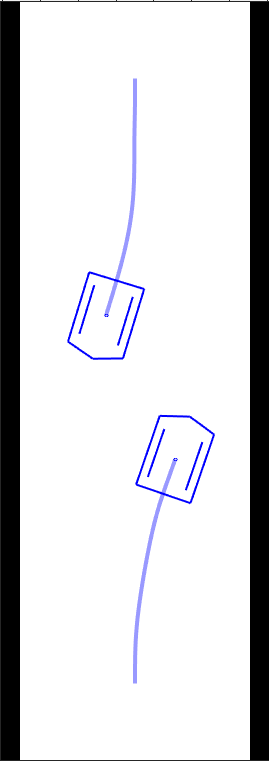}
  \caption{t=4s}
\end{subfigure}
\begin{subfigure}{0.15\textwidth}
  \centering
  \includegraphics[height=5cm,width=.75\linewidth]{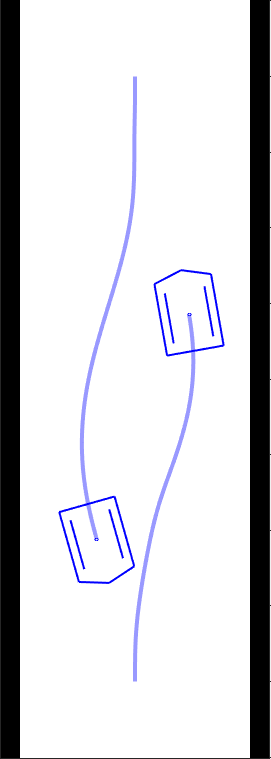}
  \caption{t=6s}
\end{subfigure}
\caption{{\bf Narrow Corridor}: We consider two robots navigating a narrow corridor in opposing directions. Figure (a)-(c) illustrates the case with two robots using MPEPC's cost function. We observe the robots eventually deadlock in this complex case. Figure (d)-(f) shows the trajectories followed by the robots for the modified cost function by DS-MPEPC. In this case, the robot navigates without colliding and deadlocking and reaches the other side of the corridor. This demonstrates the improved navigation behavior of DS-MPEPC with non-circular agents.}
\label{fig:narrowcorridor}
\end{figure*}

\begin{figure*}
\centering
\begin{subfigure}{0.24\textwidth}
  \centering
  \includegraphics[width=.85\linewidth]{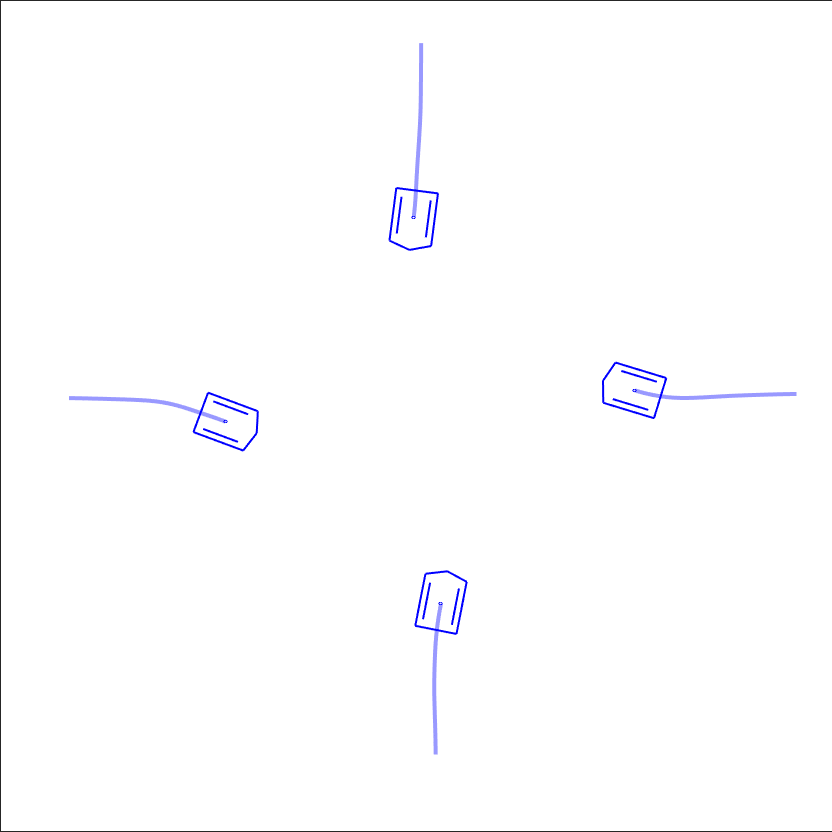}
  \caption{t = 4s}
\end{subfigure}
\begin{subfigure}{0.24\textwidth}
  \centering
  \includegraphics[width=.85\linewidth]{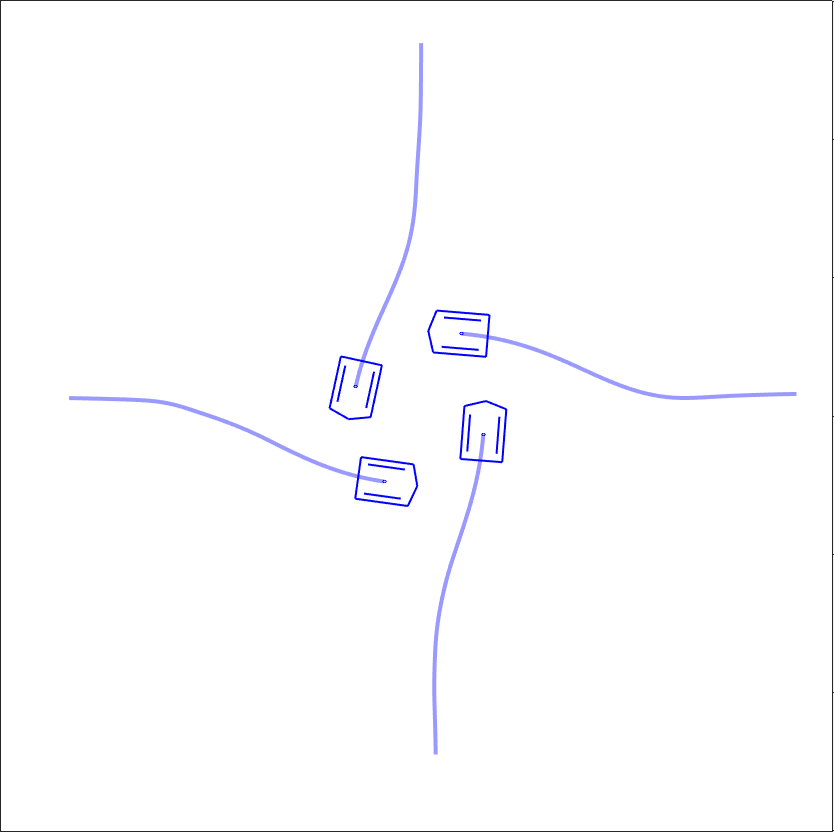}
  \caption{t = 6s}
\end{subfigure}%
\begin{subfigure}{0.24\textwidth}
  \centering
  \includegraphics[width=.85\linewidth]{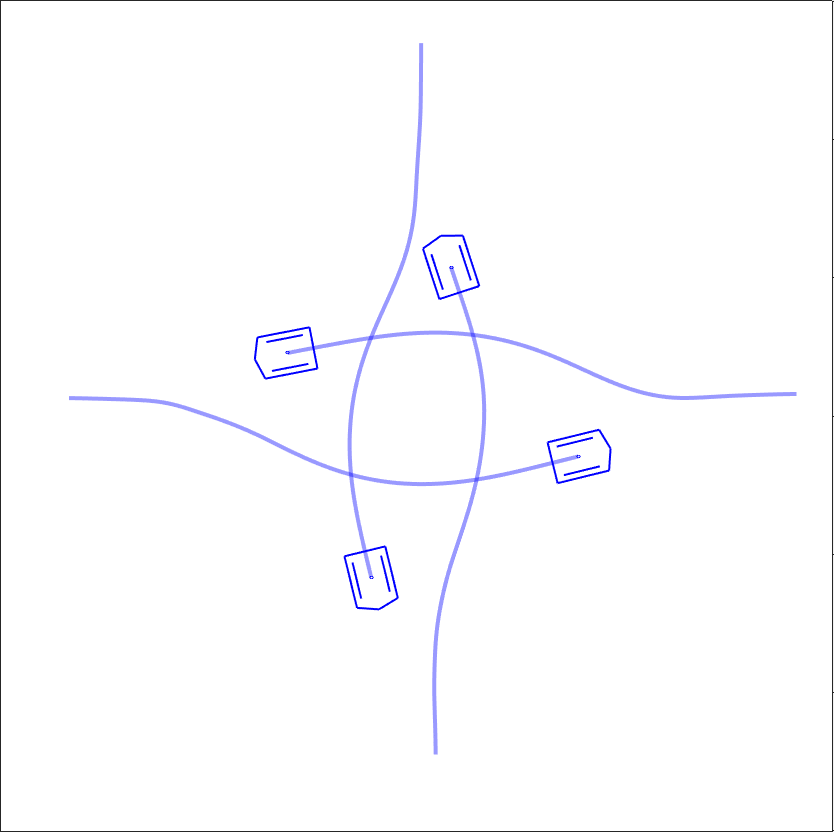}
  \caption{t = 8s}
\end{subfigure}
\begin{subfigure}{0.24\textwidth}
  \centering
  \includegraphics[width=.85\linewidth]{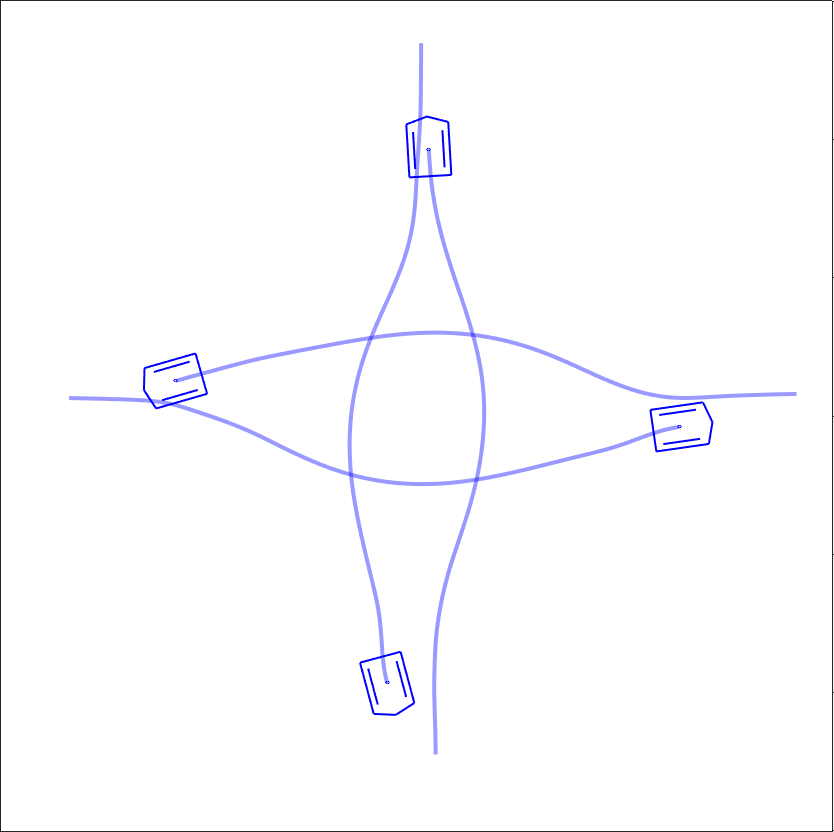}
  \caption{t = 10s}
\end{subfigure}
\caption{ {\bf Multi-Agent Benchmark}: We illustrate a scenarios with four non-circular agents arranged on the perimeter of the circle moving to their diagonally opposite position. The proposed cost function aids in navigation the robot safely in this multi-agent scenario. DS-MPEPC can generate smooth and deadlock-free paths in this scenarios.}
\label{fig:circlescenario}
\end{figure*}

\begin{figure*}
\centering
\begin{subfigure}{0.19\textwidth}
  \centering
  \includegraphics[width=.95\linewidth]{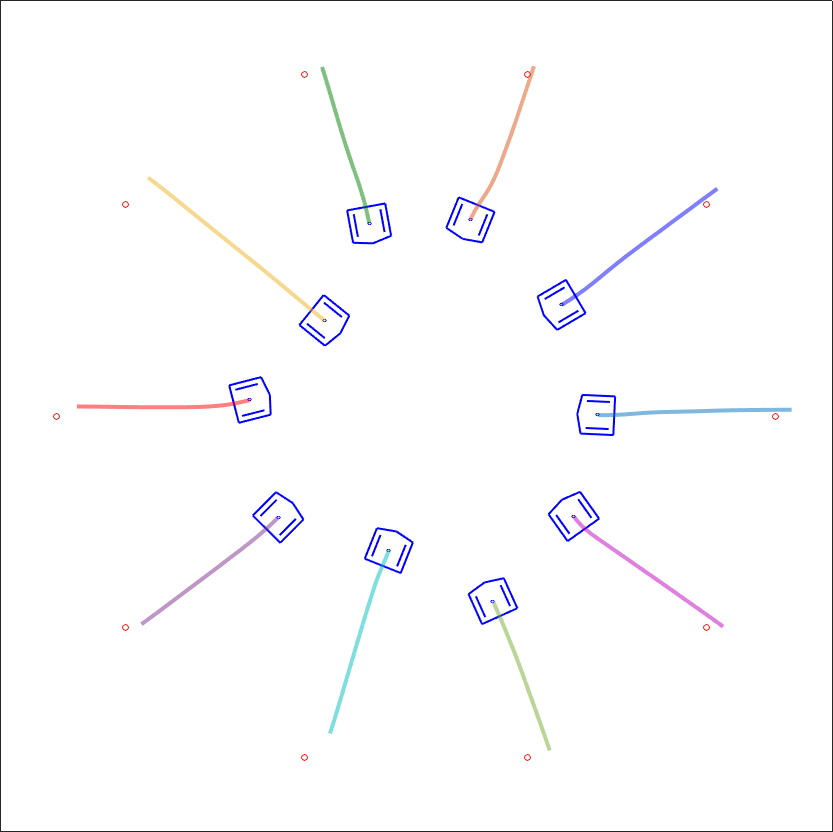}
  \caption{t = 5s}
\end{subfigure}
\begin{subfigure}{0.19\textwidth}
  \centering
  \includegraphics[width=.95\linewidth]{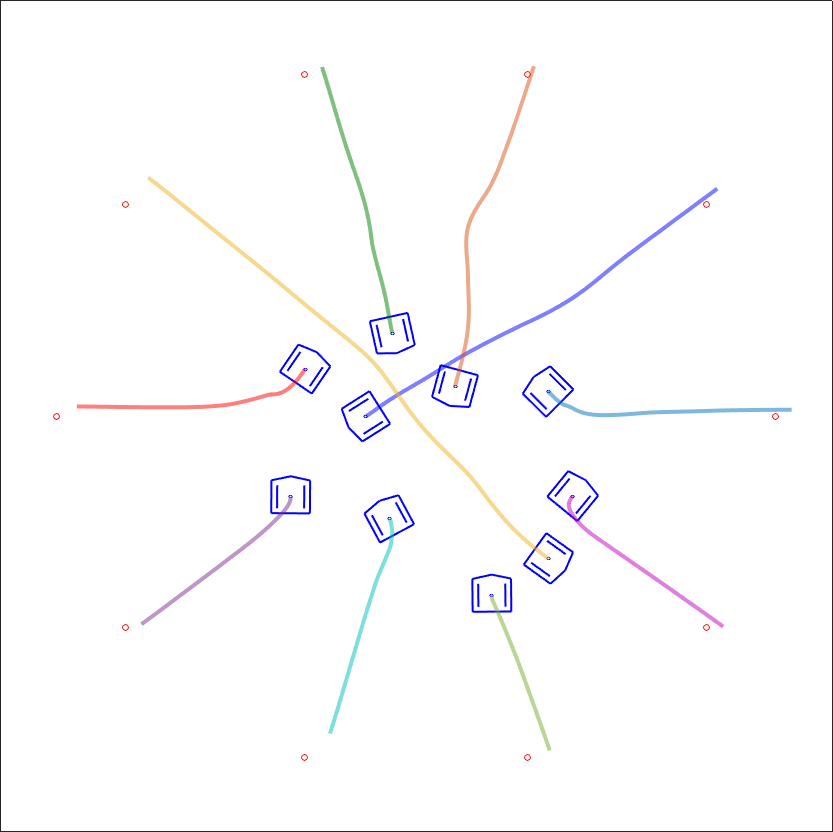}
  \caption{t = 10s}
\end{subfigure}
\begin{subfigure}{0.19\textwidth}
  \centering
  \includegraphics[width=.95\linewidth]{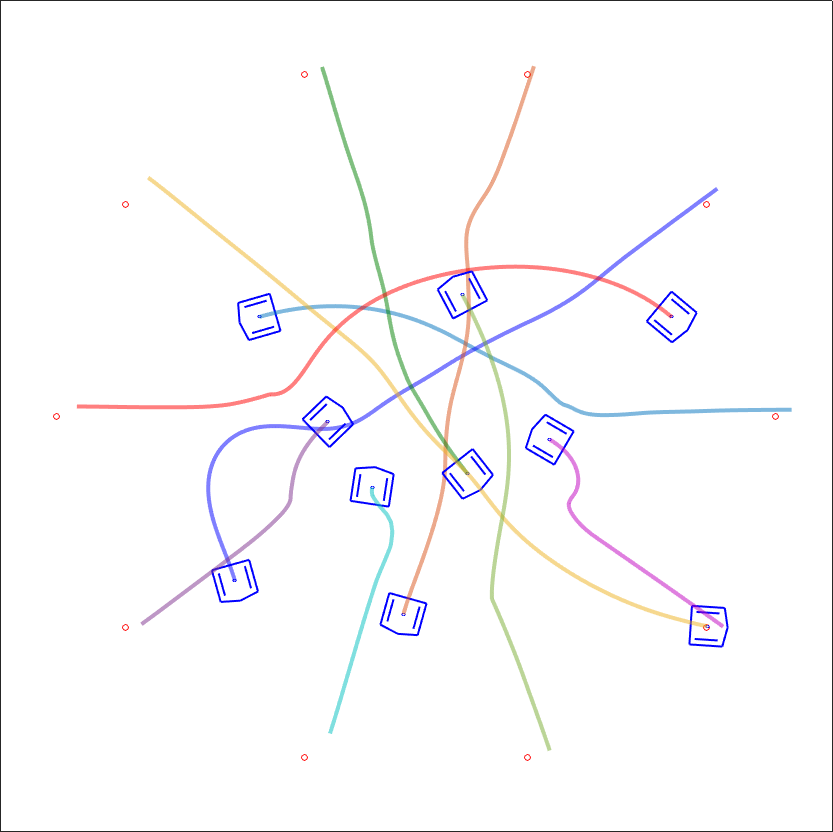}
  \caption{t = 15s}
\end{subfigure}
\begin{subfigure}{0.19\textwidth}
  \centering
  \includegraphics[width=.95\linewidth]{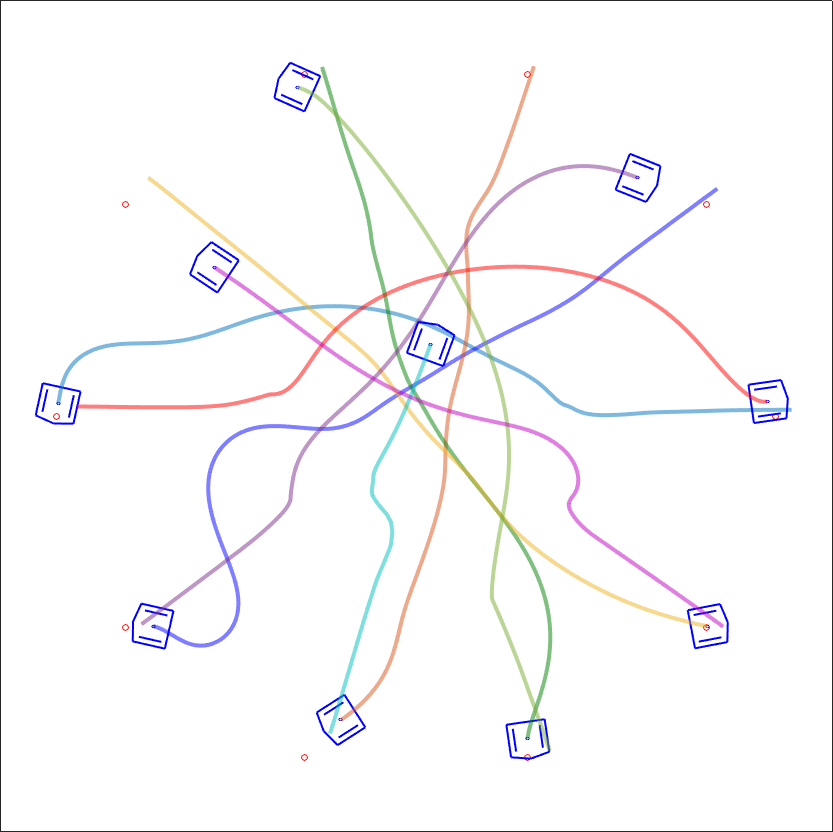}
  \caption{t = 20s}
\end{subfigure}
\begin{subfigure}{0.19\textwidth}
  \centering
  \includegraphics[width=.95\linewidth]{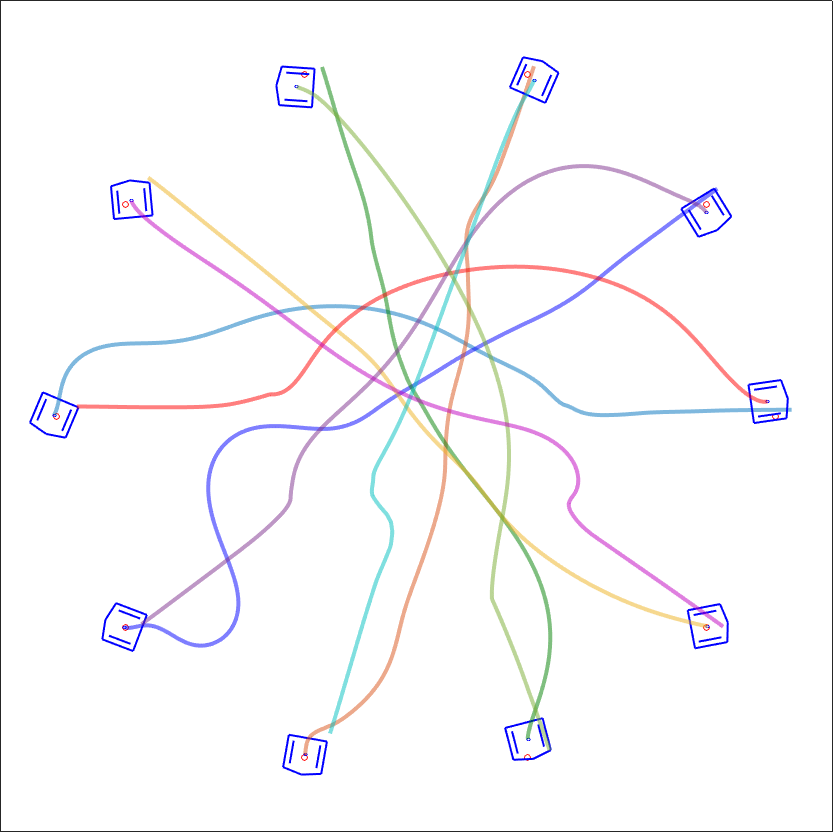}
  \caption{t = 25s}
\end{subfigure}
\caption{ {\bf Multi-Agent Benchmark}: We illustrate a scenarios with ten robots arranged on the perimeter of the circle moving to their diagonally opposite position. The proposed cost function aids in navigation the robot safely in this multi-agent scenario. DS-MPEPC can generate smooth and deadlock-free paths in this scenarios.}
\label{fig:circlescenario10}
\end{figure*}

\subsection{Evaluation Setup}
Our proposed method is implemented over the MPEPC~\cite{park_mpepc} navigation framework. Our evaluations are run in a MATLAB simulation a laptop running a $2.7$ GHz Quad-Core Intel i7 processor. Moreover, we also test the method by dynamically simulating them on environments generated from real data as in~\cite{park_mpepc}. For our evaluations, the planner uses a receding horizon of $T = 5$s with a timestep of $0.2$s. The weight parameter $a$ in collision probability ($\tilde{p}_{c_i}$) is set to $0.7$ for the evaluation. For $J_{terminal}$ computation, the $\sigma_{1/TTG} = 10^{-3}$ and $\sigma_{1/TTC} = 0.5$ are used. 

\subsection{Navigation Behavior}
We evaluate our cost function in multiple complex scenarios in simulation. In particular, we consider three indoor environments: First, a hall environment with multiple pedestrians and static obstacles, an L-shaped corridor, and a T-shaped corridor. The hall and L-shaped environments are based on real data traces, which are used to create the static obstacles and pedestrian trajectories in the simulation scenario. In the hall scenario, we increase the pedestrians from the four available pedestrian trajectories by spatially moving the pedestrian trajectories to a different portion of the environment to create additional pedestrians.

The hall environment consists of multiple pedestrians (denoted by red disks) and static obstacles (black regions). The navigation scenario involves the robot following a selected pedestrian (green disk) while avoiding collisions with other pedestrians and static obstacles. Figure~\ref{fig:hall} shows the resulting trajectory generated by our cost function.

The L-shaped corridor involves the robot following a pedestrian (green disk) into a narrow corridor. Figure~\ref{fig:lcorridor} shows the resulting trajectory generated by our cost function. We observe the robot successfully maneuvering and entering the narrow passage to follow the moving target.

The T-shaped corridor environment involves a robot turning into a corridor with the other robot staying stationary and obstructing the path. This scenario has a stationary robot blocking the moving robot, and the MPEPC cost formulation deadlocks the agent. The deadlock occurs as the halting trajectory makes the most progress while remaining safe. Our cost modification helps the agent to detour and move around the obstruction to reach the goal. Figure~\ref{fig:tcorridor} shows the resulting trajectory generated by MPEPC and our modified cost function.

\subsection{Multi-Agent Scenario}
We observe our proposed cost function was able to generate multi-agent navigation behavior in cluttered scenarios using non-circular agents. In this subsection, we evaluate our cost function in two multi-agent scenarios. In the first scenario, we consider two agents navigating a narrow corridor in opposing directions. In this particular test case (Figure~\ref{fig:narrowcorridor}), we observe the MPEPC cost to lead the robot to a deadlock, and our modified cost function navigates the robots safely.

Second, we consider a circle scenario with four-agent and ten-agents. The robots are initially on the boundary and move towards the diagonally opposite location. Figures~\ref{fig:circlescenario} and~\ref{fig:circlescenario10} illustrate the resulting trajectories in this scenario. We observe the planner navigates the robot safely and maintains a safe distance between the agents.

\subsection{Performance}
The proposed cost function involves computing the time-to-collision and time-to-goal values which are fast to compute. The optimization problem is similar to MPEPC formulation and is suitable for real-time navigation performance.
\section{Conclusion, Limitation, and Future Work}

In this paper, we consider the finite horizon trajectory planning problem for robot navigation and presented a trajectory cost function for robot navigation. Our approach extends the MPEPC navigation algorithm and considerably improves the performance in terms of avoiding deadlocks or freezing behaviors. Particularly, our proposed collision probability formulation is a function of the obstacle distance and the time to collision and is less conservative in terms of collision evaluation of the trajectory. In addition, we also propose a terminal state cost function using the time-to-goal and time-to-collision values which aid in reducing the deadlock in our evaluation scenarios. We evaluated the proposed cost function in a variety of scenarios and generates impressive navigation behavior. The overall algorithm is fast and simple and provides realtime performance in our benchmarks.

Our approach has some limitations. Some of them arise from the underlying optimization framework used for MPC or MPEPC. Though our modified cost shows improved deadlock-reducing behavior in the evaluated scenarios, it can still cause locally optimal behavior due to the finite horizon optimization. In future work, we plan to perform more evaluation of our cost function on complex simulated environments and physical robots. Besides, we plan to evaluate the terminal state cost to study the effects of the parameters $\sigma_{1/TTC}$ and $\sigma_{1/TTG}$ on the navigation and deadlocking behavior. We would like to evaluate the performance in complex real-world and synthetic scenarios.


\bibliographystyle{IEEEtran}
\bibliography{IEEEabrv,mybibfile}

\begin{thebibliography}{10}
\providecommand{\url}[1]{#1}
\csname url@rmstyle\endcsname
\providecommand{\newblock}{\relax}
\providecommand{\bibinfo}[2]{#2}
\providecommand\BIBentrySTDinterwordspacing{\spaceskip=0pt\relax}
\providecommand\BIBentryALTinterwordstretchfactor{4}
\providecommand\BIBentryALTinterwordspacing{\spaceskip=\fontdimen2\font plus
\BIBentryALTinterwordstretchfactor\fontdimen3\font minus
  \fontdimen4\font\relax}
\providecommand\BIBforeignlanguage[2]{{%
\expandafter\ifx\csname l@#1\endcsname\relax
\typeout{** WARNING: IEEEtran.bst: No hyphenation pattern has been}%
\typeout{** loaded for the language `#1'. Using the pattern for}%
\typeout{** the default language instead.}%
\else
\language=\csname l@#1\endcsname
\fi
#2}}

\bibitem{vo}
\BIBentryALTinterwordspacing
P.~Fiorini and Z.~Shiller, ``Motion planning in dynamic environments using
  velocity obstacles,'' \emph{The International Journal of Robotics Research},
  vol.~17, no.~7, pp. 760--772, 1998. [Online]. Available:
  \url{https://doi.org/10.1177/027836499801700706}
\BIBentrySTDinterwordspacing

\bibitem{rvo}
J.~van~den Berg, M.~Lin, and D.~Manocha, ``Reciprocal velocity obstacles for
  real-time multi-agent navigation,'' in \emph{2008 IEEE International
  Conference on Robotics and Automation}, 2008, pp. 1928--1935.

\bibitem{orca}
J.~Van Den~Berg, S.~J. Guy, M.~Lin, and D.~Manocha, ``Reciprocal n-body
  collision avoidance,'' in \emph{Robotics Research: The 14th International
  Symposium ISRR}.\hskip 1em plus 0.5em minus 0.4em\relax Springer, 2011, pp.
  3--19.

\bibitem{bvc}
D.~Zhou, Z.~Wang, S.~Bandyopadhyay, and M.~Schwager, ``Fast, on-line collision
  avoidance for dynamic vehicles using buffered voronoi cells,'' \emph{IEEE
  Robotics and Automation Letters}, vol.~2, no.~2, pp. 1047--1054, 2017.

\bibitem{mpc_orca}
H.~Cheng, Q.~Zhu, Z.~Liu, T.~Xu, and L.~Lin, ``Decentralized navigation of
  multiple agents based on orca and model predictive control,'' in \emph{2017
  IEEE/RSJ International Conference on Intelligent Robots and Systems (IROS)},
  2017, pp. 3446--3451.

\bibitem{brito_mpcc}
B.~Brito, B.~Floor, L.~Ferranti, and J.~Alonso-Mora, ``Model predictive
  contouring control for collision avoidance in unstructured dynamic
  environments,'' \emph{IEEE Robotics and Automation Letters}, vol.~4, no.~4,
  pp. 4459--4466, 2019.

\bibitem{park_mpepc}
J.~J. Park, C.~Johnson, and B.~Kuipers, ``Robot navigation with model
  predictive equilibrium point control,'' in \emph{2012 IEEE/RSJ International
  Conference on Intelligent Robots and Systems}, 2012, pp. 4945--4952.

\bibitem{epc}
A.~Jain and C.~C. Kemp, ``Pulling open novel doors and drawers with equilibrium
  point control,'' in \emph{2009 9th IEEE-RAS International Conference on
  Humanoid Robots}, 2009, pp. 498--505.

\bibitem{pfm}
Y.~Koren and J.~Borenstein, ``Potential field methods and their inherent
  limitations for mobile robot navigation,'' in \emph{Proceedings. 1991 IEEE
  International Conference on Robotics and Automation}, 1991, pp. 1398--1404
  vol.2.

\bibitem{best2016real}
A.~Best, S.~Narang, and D.~Manocha, ``Real-time reciprocal collision avoidance
  with elliptical agents,'' in \emph{2016 IEEE International Conference on
  Robotics and Automation (ICRA)}.\hskip 1em plus 0.5em minus 0.4em\relax IEEE,
  2016, pp. 298--305.

\bibitem{he2017efficient}
L.~He, J.~Pan, and D.~Manocha, ``Efficient multi-agent global navigation using
  interpolating bridges,'' in \emph{2017 IEEE International Conference on
  Robotics and Automation (ICRA)}.\hskip 1em plus 0.5em minus 0.4em\relax IEEE,
  2017, pp. 4391--4398.

\bibitem{avo}
J.~van~den Berg, J.~Snape, S.~J. Guy, and D.~Manocha, ``Reciprocal collision
  avoidance with acceleration-velocity obstacles,'' in \emph{2011 IEEE
  International Conference on Robotics and Automation}, 2011, pp. 3475--3482.

\bibitem{lqr}
D.~Bareiss and J.~van~den Berg, ``Reciprocal collision avoidance for robots
  with linear dynamics using lqr-obstacles,'' in \emph{2013 IEEE International
  Conference on Robotics and Automation}, 2013, pp. 3847--3853.

\bibitem{lqg}
J.~van~den Berg, D.~Wilkie, S.~J. Guy, M.~Niethammer, and D.~Manocha,
  ``Lqg-obstacles: Feedback control with collision avoidance for mobile robots
  with motion and sensing uncertainty,'' in \emph{2012 IEEE International
  Conference on Robotics and Automation}, 2012, pp. 346--353.

\bibitem{orca_dd}
J.~Snape, J.~van~den Berg, S.~J. Guy, and D.~Manocha, ``Smooth and
  collision-free navigation for multiple robots under differential-drive
  constraints,'' in \emph{2010 IEEE/RSJ International Conference on Intelligent
  Robots and Systems}, 2010, pp. 4584--4589.

\bibitem{dwa}
D.~Fox, W.~Burgard, and S.~Thrun, ``The dynamic window approach to collision
  avoidance,'' \emph{IEEE Robotics \& Automation Magazine}, vol.~4, no.~1, pp.
  23--33, 1997.

\bibitem{nh-ttc}
B.~Davis, I.~Karamouzas, and S.~J. Guy, ``Nh-ttc: A gradient-based framework
  for generalized anticipatory collision avoidance,'' \emph{arXiv preprint
  arXiv:1907.05945}, 2019.

\bibitem{cglr}
S.~H. Arul and D.~Manocha, ``Cglr: Dense multi-agent navigation using voronoi
  cells and congestion metric-based replanning,'' in \emph{2022 IEEE/RSJ
  International Conference on Intelligent Robots and Systems (IROS)}, 2022, pp.
  7213--7220.

\bibitem{ics}
T.~Fraichard and H.~Asama, ``Inevitable collision states. a step towards safer
  robots?'' in \emph{Proceedings 2003 IEEE/RSJ International Conference on
  Intelligent Robots and Systems (IROS 2003) (Cat. No.03CH37453)}, vol.~1,
  2003, pp. 388--393 vol.1.

\bibitem{trautman}
P.~Trautman, J.~Ma, R.~M. Murray, and A.~Krause, ``Robot navigation in dense
  human crowds: the case for cooperation,'' in \emph{2013 IEEE International
  Conference on Robotics and Automation}, 2013, pp. 2153--2160.

\bibitem{cadrl}
Y.~F. Chen, M.~Liu, M.~Everett, and J.~P. How, ``Decentralized
  non-communicating multiagent collision avoidance with deep reinforcement
  learning,'' in \emph{2017 IEEE International Conference on Robotics and
  Automation (ICRA)}, 2017, pp. 285--292.

\bibitem{long}
P.~Long, T.~Fan, X.~Liao, W.~Liu, H.~Zhang, and J.~Pan, ``Towards optimally
  decentralized multi-robot collision avoidance via deep reinforcement
  learning,'' in \emph{2018 IEEE International Conference on Robotics and
  Automation (ICRA)}, 2018, pp. 6252--6259.

\bibitem{densecavoid}
A.~J. Sathyamoorthy, J.~Liang, U.~Patel, T.~Guan, R.~Chandra, and D.~Manocha,
  ``Densecavoid: Real-time navigation in dense crowds using anticipatory
  behaviors,'' in \emph{2020 IEEE International Conference on Robotics and
  Automation (ICRA)}, 2020, pp. 11\,345--11\,352.

\bibitem{rl_pednav}
C.~Pérez-D’Arpino, C.~Liu, P.~Goebel, R.~Martín-Martín, and S.~Savarese,
  ``Robot navigation in constrained pedestrian environments using reinforcement
  learning,'' in \emph{2021 IEEE International Conference on Robotics and
  Automation (ICRA)}, 2021, pp. 1140--1146.

\bibitem{vrvo}
S.~H. Arul and D.~Manocha, ``V-rvo: Decentralized multi-agent collision
  avoidance using voronoi diagrams and reciprocal velocity obstacles,'' in
  \emph{2021 IEEE/RSJ International Conference on Intelligent Robots and
  Systems (IROS)}, 2021, pp. 8097--8104.

\bibitem{wallfollowing}
G.~M. Sanchez and L.~L. Giovanini, ``Autonomous navigation with deadlock
  detection and avoidance,'' 2014.

\bibitem{park_smoothlaw}
J.~J. Park and B.~Kuipers, ``A smooth control law for graceful motion of
  differential wheeled mobile robots in 2d environment,'' in \emph{2011 IEEE
  International Conference on Robotics and Automation}, 2011, pp. 4896--4902.

\bibitem{park_thesis}
J.~J. Park, ``Graceful navigation for mobile robots in dynamic and uncertain
  environments.'' Ph.D. dissertation, 2016.

\end{thebibliography}

\end{document}